\documentclass{article}

\usepackage[preprint]{neurips_2026}

\usepackage[utf8]{inputenc} 
\usepackage[T1]{fontenc}    
\usepackage{hyperref}       
\usepackage{url}            
\usepackage{booktabs}       
\usepackage{amsfonts}       
\usepackage{nicefrac}       
\usepackage{microtype}      
\usepackage{xcolor}         
\usepackage{mymacro}

\title{Differentially Private Sampling from Distributions \\ via Wasserstein Projection}

\author{%
  Shokichi Takakura \\
  LY Corporation \\
  \texttt{stakakur@lycorp.co.jp} \\
  \And
  Seng Pei Liew \\
  LY Corporation \\
  \texttt{sliew@lycorp.co.jp} \\
  \AND
  Satoshi Hasegawa \\
  LY Corporation \\
  \texttt{satoshi.hasegawa@lycorp.co.jp} 
}

\begin{document}

\maketitle

\begin{abstract}
  In this paper, we study the problem of sampling from a distribution under the constraint of differential privacy (DP).
  Prior works measure the utility of DP sampling with density ratio-based measures such as KL divergence.
  However, such formulations suffer from two key limitations: 1) they fail to capture the geometric structure of the support, and 2) they are not applicable when the supports of the distributions differ.
  To deal with these issues, we develop a novel framework for DP sampling with Wasserstein distance as the utility measure.
  In this formulation, we propose \emph{Wasserstein Projection Mechanism (WPM)}, a minimax optimal mechanism based on Wasserstein projection.
  Furthermore, we develop efficient algorithms for computing the proposed mechanisms approximately and provide convergence guarantees.
\end{abstract}

\section{Introduction}

Differential privacy (DP)~\citep{dwork2006calibrating,dwork2014algorithmic} provides a rigorous framework for quantifying and ensuring privacy in data analysis and machine learning.
Especially, local differential privacy (LDP)~\citep{kasiviswanathan2011can} is a strong notion of privacy that does not require a trusted server and thus has been widely adopted in practice~\citep{erlingsson2014rappor,apple2017dp}.

However, most of the prior work on LDP assumes that each user holds only a single data point,
which is unrealistic since modern devices often collect and store a large amount of data~\citep{husain2020local}.
Several works~\citep{levy2021learning,ghazi2021user} consider the user-level privacy by assuming that each user holds a dataset of the same size generated from an underlying distribution.
Although this scenario is more realistic, the requirement of equal dataset size across users is still restrictive.

To mitigate this issue, a line of work~\citep{husain2020local,park2024exactly,zamanlooy2025locally,ghoukasian2025locally} has treated each user's dataset as a probability distribution
and proposed differentially private mechanisms, which take a distribution as input and release a single sample.
We refer to this problem as \emph{private sampling}.
Some of these works have derived optimal mechanisms in terms of worst-case $f$-divergence between the input and output distributions
for various settings~\citep{park2024exactly,zamanlooy2025locally,ghoukasian2025locally}.

While these works have provided valuable insights into the design of private sampling mechanisms,
most of them rely on density ratio-based measures such as KL divergence to quantify the utility of the mechanism.
However, such formulations suffer from two key limitations:
\begin{itemize}
	\item \emph{Support Mismatch}: Metrics based on density ratios are not meaningful when the supports of the distributions differ.
	\item \emph{Geometric Ignorance}: They do not capture the geometric structure of the underlying space (support), which can be crucial for downstream applications.
\end{itemize}
For example, location data is often represented as points in a metric space~\citep{chatzikokolakis2015constructing},
and in recommendation systems and natural language processing,
items and words are typically embedded in a latent space where distances capture semantic similarity~\citep{koren2009matrix,mikolov2013efficient},
making the geometry of the space crucial for evaluating utility.
Moreover, support mismatch arises naturally in these settings:
geographic locations may be aggregated to representative points~\citep{chatzikokolakis2015constructing},
a recommender may sample from a broader catalog beyond observed items~\citep{ricci2010introduction},
and language models often restrict or modify the candidate set (e.g., via truncation or special tokens)~\citep{vinod2025invisibleink}.

To address these limitations, we propose a novel framework for private sampling that optimizes
the worst-case \textit{Wasserstein distance}~\citep{villani2009optimal} between the privatized distribution of the mechanism and the target distribution.
Since Wasserstein distance captures the geometric structure of the support and is properly defined even when the supports of the distributions differ,
our framework provides a more flexible way to design private sampling mechanisms.
Furthermore, we develop a \emph{Wasserstein Projection Mechanism (WPM)} based on Wasserstein projection,
which is shown to be minimax optimal in terms of worst-case Wasserstein distance.
Since computing the exact projection can be computationally expensive,
we propose a practical algorithm for computing the approximate projection using entropic regularization and alternating projections.
We provide linear convergence guarantees and approximation error bounds for our algorithm.
Finally, we demonstrate the effectiveness of our proposed method through numerical experiments on synthetic and real-world datasets.

Our contributions are as follows:
\begin{itemize}
	\item To mitigate the limitations of density ratio-based formulations, we propose a novel framework for designing private sampling mechanisms by utilizing Wasserstein distance as the utility measure.
	\item We develop \emph{Wasserstein Projection Mechanism} (WPM), a minimax optimal mechanism in terms of worst-case Wasserstein distance.
	\item We propose practical algorithms for computing the optimal mechanism using entropic regularized Wasserstein distance and alternating projections.
	      By exploiting the special structure of the problem, we derive a linear convergence guarantee for our algorithm.
\end{itemize}

\subsection{Related Work}
\paragraph{Statistics Estimation in LDP}
A rich literature studies estimating statistics from data privatized under local differential privacy~\citep{asi2022optimal},
including frequency/histogram and mean estimation.
Foundationally, \citet{duchi2013local} derived minimax rates for several estimation tasks under LDP,
and \citet{kairouz2014extremal} characterized extremal LDP mechanisms.
In contrast, our focus is LDP \emph{sampling} mechanisms that control distributional distortion (here, in Wasserstein distance) rather than estimating a fixed statistic.
This can be viewed as a generalization of the classical estimation problem, where the input distribution is a Dirac measure at a single point.
See Appendix~\ref{app:related-work} for a more detailed discussion.

\paragraph{Wasserstein Distance in DP}
Several works have utilized the notion of Wasserstein distance in the context of differential privacy~\citep{feldman2024instance,rodriguez2025learning,gu2025differentially}.
For instance, \citet{le2019differentially,rakotomamonjy2021differentially} proposed randomized mechanisms to compute the (sliced) Wasserstein distance privately.
\citet{feldman2024instance} utilized Wasserstein distance as a utility measure for private density estimation.
In contrast to our work, which considers privacy leakage of a \emph{sample} from the distribution, in their scenario, the estimated distribution itself is the output of the mechanism.
Therefore, while these works have utilized similar terminology,
their focus is completely different from ours.

\subsection{Notation}
Let $\R_{+}$ ($\R_{++}$) denote the set of non-negative (positive) real numbers,
and $[k] = \{1, 2, \dots, k\}$ for $k \in \N$.
For $x \in \mathcal{X}$, we denote by $\delta_x$ the Dirac measure at $x$, i.e., $\delta_x(A) = 1$ if $x \in A$ and $\delta_x(A) = 0$ otherwise.
For $a \in \R$, we denote by $(a)_+$ the positive part of $a$, i.e., $(a)_+ = \max\{a, 0\}$.

\section{Preliminaries}
\subsection{Differentially Private Sampling}
Differential privacy is a mathematically rigorous framework for quantifying the privacy guarantees of algorithms that operate on sensitive data.
Let $\mathcal{X}$ be a measurable space equipped with a $\sigma$-algebra $\mathcal{F}$, and
let $\mathcal{P}(\mathcal{X})$ be the set of all probability measures on $\mathcal{X}$.
In the local model of differential privacy (LDP),
a client has an input distribution $\mu \in \mathcal{P}(\mathcal{X})$ which depends on sensitive data.
Therefore, sampling directly from $\mu$ can potentially leak sensitive information about the data.
To ensure privacy, a client samples from a transformed distribution $\mathcal{M}[\mu]$,
where $\mathcal{M}$ is a functional mapping $\mathcal{P}(\mathcal{X})$ to $\mathcal{P}(\mathcal{X})$.
We require that $\mathcal{M}$ satisfy the constraints of differential privacy, which ensures that a sample from the distribution $\mathcal{M}[\mu]$ does not reveal too much information about the input distribution $\mu$
while preserving as much utility as possible.
Formally, we have the following definitions of LDP.
\begin{definition}[Local Differential Privacy (LDP)]
	A mechanism $\mathcal{M}$ is said to satisfy $\varepsilon$-differential privacy for some $\varepsilon > 0$,
	if for all pairs of input distributions $\mu, \mu' \in \mathcal{P}(\mathcal{X})$
	and for all measurable subsets $S \subseteq \mathcal{X}$, we have:
	\begin{align*}
		\mathcal{M}[\mu](S) \leq e^{\varepsilon} \mathcal{M}[\mu'](S).
	\end{align*}
\end{definition}
Conventionally, $\varepsilon$-DP is defined for randomized mechanisms and here,
we identify a functional $\mathcal{M}: \mathcal{P}(\mathcal{X}) \to \mathcal{P}(\mathcal{X})$ with a randomized mechanism that samples from $\mathcal{M}[\mu]$.

\subsection{KL Projection Mechanism}
Here, we briefly review the $f$-divergence based mechanism proposed by~\citet{husain2020local,park2024exactly}, which we refer to as KL projection mechanism (KPM),
for private sampling under LDP.
For a convex function $f: \R_{+} \to \R$ with $f(1) = 0$, the $f$-divergence between two probability measures $\mu$ and $\nu$ is defined as
\begin{align*}
	D_f(\mu \| \nu) = \int_{\mathcal{X}} f\left(\frac{d\mu}{d\nu}\right) d\nu,
\end{align*}
where $\frac{d\mu}{d\nu}$ is the Radon-Nikodym derivative of $\mu$ with respect to $\nu$.
Common examples of $f$-divergences include the Kullback-Leibler (KL) divergence and $\chi^2$-divergence.

In the discrete setting $\mathcal{X}=[k]$ for some $k \in \N$,
the KL projection mechanism~\citep{park2024exactly} is defined as follows:
\begin{align*}
	\mathcal{M}^{\kl}[\mu](x) = \max\ab(\frac{1}{r_\mu} \mu(x), \frac{1}{e^\varepsilon + k - 1}),
\end{align*}
where $r_\mu > 0$ is a normalization constant so that $\mathcal{M}^{\kl}[\mu]$ is a probability distribution.
This is minimax optimal in terms of worst-case $f$-divergence $\sup_\mu D_f(\mu \| \mathcal{M}[\mu])$ for any $f$-divergence~\citep{park2024exactly}.
Furthermore, this can be regarded as a projection of $\mu$ with respect to KL divergence onto the LDP polytope with base measure $m(x) = \frac{e^{\varepsilon/2}}{e^{\varepsilon} + k - 1}$ defined as follows:
\begin{definition}[LDP Polytope]
	For a privacy parameter $\varepsilon$, the LDP polytope $Q_{m, \varepsilon}$ is defined as
	\begin{align}
		Q_{m, \varepsilon} = \{\nu \in \mathcal{P}(\mathcal{X}) : e^{-\varepsilon / 2} m(A) \leq \nu(A) \leq e^{\varepsilon / 2} m(A) \text{ for all measurable } A \subseteq \mathcal{X} \}, \label{eq:ldp-polytope}
	\end{align}
	where $m$ is a base measure on $\mathcal{X}$. Note that $m$ is not necessarily a probability measure.
\end{definition}
In the following, we only consider base measures $m$ such that $Q_{m,\varepsilon} \neq \emptyset$,
i.e., $e^{-\varepsilon/2}\, m(\mathcal{X}) \leq 1 \leq e^{\varepsilon/2}\, m(\mathcal{X})$.
It is not straightforward to extend this mechanism to continuous settings because $f$-divergences can be infinite when $\mu$ is not absolutely continuous with respect to $\nu$.
Prior works deal with this issue by restricting the input space of the mechanism. See~\citet{park2024exactly} for details.

While $f$-divergences are widely used in the literature to measure the distance between probability distributions,
they have several limitations.
First, $f$-divergences are not meaningful when the supports of the distributions differ.
For instance, in the case where we can only sample from a restricted output space $\mathcal{V} \subset \mathcal{X}$,
$\mu$ may not be absolutely continuous with respect to $\mathcal{M}[\mu]$ and thus, the Radon-Nikodym derivative $\frac{d\mu}{d\mathcal{M}[\mu]}$ does not exist
\footnote{Total variation distance can be defined even when the supports differ but discrepancies outside the common support mainly capture support mismatch and are not substantively meaningful.}.
Second, they do not capture the geometric structure of the underlying space $\mathcal{X}$.
Let us consider the case where the input distribution $\mu$ is a Dirac measure at a point $x \in \mathcal{X}$, i.e., $\mu = \delta_x$.
In this case, KPM assigns a positive mass $\frac{1}{e^\varepsilon + k - 1}$ to all points except $x$ including points \emph{far} from $x$ in a certain metric, which can lead to a large distortion.

\section{Proposed Method: Wasserstein Projection Mechanism}
To address these limitations of KPM, we propose a novel framework for designing private sampling mechanisms by utilizing Wasserstein distance as the utility measure.
Then, we develop a Wasserstein Projection Mechanism (WPM) based on Wasserstein projection, which is shown to be minimax optimal in terms of worst-case Wasserstein distance.

\subsection{Wasserstein Distance}
In this paper, we assume that the underlying space $\mathcal{X}$ is equipped with a metric $d: \mathcal{X} \times \mathcal{X} \to \R_{+}$.
For instance, if $\mathcal{X}$ is a subset of $\R^d$, we can use the Euclidean distance as the metric.
In such a case, the Wasserstein distance is a natural choice for measuring the distance between probability distributions on $\mathcal{X}$, as it captures the geometric structure of the space.
For any $\mu \in \mathcal{P}(\mathcal{X})$ and $\nu \in \mathcal{P}(\mathcal{X})$, the Wasserstein distance of order $p$ is defined as
\begin{align*}
	W_p(\nu, \mu) = \left( \inf_{\pi \in \Pi(\nu, \mu)} \int_{\mathcal{X} \times \mathcal{X}} d(x, y)^p \, d\pi(x, y) \right)^{1/p},
\end{align*}
where $\Pi(\nu, \mu)$ is the set of all couplings of $\nu$ and $\mu$, i.e., the set of all probability measures on $\mathcal{X} \times \mathcal{X}$ with marginals $\mu$ and $\nu$.
For simplicity, we assume that $\mathcal{X}$ is a compact Polish space, which ensures that the Wasserstein distance is properly defined and finite for all probability measures supported on $\mathcal{X}$.
This is in contrast to density ratio-based divergences (e.g., KL), which can be infinite when the supports differ.

\subsection{Generalized Problem Formulation}
In this paper, we generalize the problem formulation of prior works by
allowing the supports of the input and output distributions to differ.
Specifically, let $\mathcal{V} \subseteq \mathcal{X}$ be a measurable subset of $\mathcal{X}$ representing the output space of the mechanism $\mathcal{M}$.
Then, we aim to design a mechanism $\mathcal{M}:\mathcal{P}(\mathcal{X}) \to \mathcal{P}(\mathcal{V})$ that satisfies $\varepsilon$-LDP
and minimizes the Wasserstein distance between $\mathcal{M}[\mu]$ and $\mu$.
In particular, we consider the following worst-case cost:
\begin{align*}
	U[\mathcal{M}] = \sup_{\mu \in \mathcal{P}(\mathcal{X})} W_p(\mathcal{M}[\mu], \mu).
\end{align*}

This general formulation allows us to consider a sampling mechanism from a restricted output space $\mathcal{V}$, which can be useful in practice when we want
to limit the range of the output distribution for computational or privacy reasons.
Note that this general formulation cannot be handled with density ratio-based formulations,
since $f$-divergences do not provide a meaningful distance when the supports of the distributions differ.

\subsection{Wasserstein Projection Mechanism}
To solve the above problem, we propose a \emph{Wasserstein Projection Mechanism (WPM)} defined as follows:
\begin{definition}[Wasserstein Projection Mechanism]
	Wasserstein Projection Mechanism with base measure $m$ and privacy parameter $\varepsilon$ is defined as
	\begin{align*}
		\mathcal{M}_{m, \varepsilon}[\mu] \in \arg\min_{\nu \in Q_{m, \varepsilon}} W_p(\nu, \mu),
	\end{align*}
	where $Q_{m, \varepsilon}$ is the LDP polytope defined as in Eq.~\eqref{eq:ldp-polytope}.
\end{definition}
Since $Q_{m, \varepsilon}$ is convex and closed, a minimizer exists under our assumptions.

This mechanism satisfies $\varepsilon$-LDP as shown in the following proposition:
\begin{proposition}\label{prop:projection-ldp}
	The Wasserstein Projection Mechanism $\mathcal{M}_{m, \varepsilon}$ satisfies $\varepsilon$-LDP.
\end{proposition}
See Appendix~\ref{app:proof-ldp} for the proof.

Furthermore, we can show that given any $\varepsilon$-LDP mechanism $\mathcal{M}$,
there exists a Wasserstein Projection Mechanism $\mathcal{M}_{m, \varepsilon}$ that has a uniformly better or equal utility across all input distributions $\mu$.
\begin{proposition}[Uniform Optimality of WPM]\label{thm:projection}
	For any $\varepsilon$-LDP mechanism $\mathcal{M}$, there exists a base measure $m$ such that for all $\mu \in \mathcal{P}(\mathcal{X})$,
	\begin{align*}
		W_p(\mathcal{M}[\mu],\mu) \geq W_p(\mathcal{M}_{m,\varepsilon}[\mu],\mu).
	\end{align*}
\end{proposition}
See Appendix~\ref{app:proof-projection} for the proof.
The key ingredient of this proof is the fact that $\mathcal{M}$ satisfies $\varepsilon$-LDP
is equivalent to the range of $\mathcal{M}$ being contained in LDP polytope $Q_{m,\varepsilon}$ for some base measure $m$.
While prior work~\citep{husain2020local,park2024exactly} utilizes the LDP polytope as a sufficient condition for satisfying LDP,
we show that this is also a necessary condition.

This theorem implies a strong optimality guarantee for the WPM $\mathcal{M}_{m,\varepsilon}$:
given any $\varepsilon$-LDP mechanism $\mathcal{M}$, we can construct a WPM $\mathcal{M}_{m,\varepsilon}$ that has a uniformly better or equal utility across all input distributions $\mu$.
In other words, as long as the Wasserstein distance is used as the utility, there is no point in using other mechanisms.

\subsection{Minimax Optimality and Optimal Base Measure}
In the following, we focus on the finite discrete case,
i.e., $\mathcal{X}=[k]$ for some $k \in \N$.
Without loss of generality, we assume that $\mathcal{V} = [k_v]$ for some $k_v \leq k$.
Due to the uniform optimality guarantee of the WPM,
WPM with the optimal base measure $m^*$ is minimax optimal in the sense that it minimizes the worst-case cost across all $\varepsilon$-LDP mechanisms.
\begin{theorem}[Minimax Optimality]\label{thm:minimax}
	There exists a base measure $m^*$ such that
	\begin{align*}
		U[\mathcal{M}] \geq U[\mathcal{M}_{m^*,\varepsilon}]
	\end{align*}
	for any $\varepsilon$-LDP mechanism $\mathcal{M}$.
\end{theorem}
See Appendix~\ref{app:proof-minimax} for the proof.

The optimal base measure $m^*$ can be obtained by optimizing the worst-case cost $U[\mathcal{M}_{m, \varepsilon}]$.
In the case of $f$-divergences, the optimal base measure is uniform over the support of the target distribution
since they are agnostic to the geometry of the space $\mathcal{X}$.
On the other hand, the optimal base measure for WPM depends on the geometry of the space $\mathcal{X}$, and thus,
finding the optimal base measure is a non-trivial problem.

Since evaluating the worst-case cost $U[\mathcal{M}]$ involves 1) calculating the Wasserstein distance for each input distribution $\mu$ and 2) taking the supremum over all input distributions $\mu$,
it is computationally intractable to directly solve the above minimax optimization problem to find the optimal base measure.
To mitigate this issue, we show that $U[\mathcal{M}_{m, \varepsilon}]$ admits a tractable reformulation
and interestingly, it is convex in $m$, which allows us to efficiently find the optimal base measure via convex optimization techniques.

\begin{theorem}[Simpler Reformulation]\label{thm:base-measure-minimax}
	Let $C_{ij} = d(i, j)^p$ for $i\in[k], j\in[k_v]$, and $\alpha := e^{-\varepsilon/2}, \beta := e^{\varepsilon/2}$.
	For each $m\in\R_+^{k_v}~(1/\beta \leq \sum_{j=1}^{k_v} m_j \leq 1/\alpha)$ and $i \in [k]$, define
	\begin{align*}
		\phi_i(m)
		 & = \tau_i(m) + \sum_{j=1}^{k_v} (\alpha(C_{ij}-\tau_i(m))_+ - \beta(\tau_i(m)-C_{ij})_+) m_j,
	\end{align*}
	where $\tau_i(m) := \min\{t: \sum_{j=1}^{k_v} (\beta m_j \mathbf{1}_{\{C_{ij} \leq t\}} + \alpha m_j \mathbf{1}_{\{C_{ij}>t\}}) \ge 1\}$.
	Then, $f(m) := \max_{i\in[k]} \phi_i(m)$ is a convex function of $m$, and for any $\varepsilon$-LDP mechanism $\mathcal{M}$,
	we have
	\begin{align*}
		U[\mathcal{M}] \geq f(m^*)^{1/p} = U[\mathcal{M}_{m^*,\varepsilon}],
	\end{align*}
	where $m^* \in \arg\min_{m\in\R_+^{k_v}:1/\beta \leq \sum_j m_j \leq 1/\alpha} f(m)$.
\end{theorem}
See Appendix~\ref{app:proof-base-measure-minimax} for the proof.
This reduction is based on the following key observations:
1) the worst-case utility is achieved by a Dirac measure at some point $i \in [k]$,
2) the optimal projection cost for a Dirac measure has a simple closed-form expression.
Once the costs $\{C_{ij}\}_{j=1}^{k_v}$ are sorted for each $i$, the function $\phi_i(m)$ can be evaluated in $O(k)$ time
and thus, $f(m)$ can be evaluated in $O(k^2)$ time.
In addition, $f(m)$ is a convex function of $m$, and thus, the optimal base measure can be obtained
via standard mirror descent algorithms.
From a standard argument in convex optimization, we can obtain an $\varepsilon_{\mathrm{opt}}$-accurate solution
in $O(\log k/\varepsilon_{\mathrm{opt}}^2)$ iterations.
Algorithmic details and convergence analysis are given in Appendix~\ref{app:mirror-descent-base-measure}.

\paragraph{Extension to continuous spaces}
While we focus on the finite discrete case for computational tractability,
our framework is not limited to discrete spaces and can be applied to continuous settings as well.
As a concrete example, we derive the optimal base measure when $\mathcal{X}$ is the $d$-dimensional unit sphere $\mathcal{S}^d$ equipped with the Euclidean distance
by exploiting the rotational symmetry of the sphere.
See  Proposition~\ref{prop:sphere} in Appendix~\ref{app:proof-sphere} for details.

\section{Practical Implementation}\label{sec:practical-implementation}
Given a base measure $m$ and privacy parameter $\varepsilon$, computing the Wasserstein projection $\mathcal{M}_{m, \varepsilon}[\mu]$
for a given input distribution $\mu$ can be formulated as the following Linear Program (LP):
\begin{align*}
	\min_{\nu \in \R_+^{k_v}, \, \pi \in \R_+^{k\times k_v}} \sum_{i=1}^{k}\sum_{j=1}^{k_v} C_{ij} \pi_{ij}  \text{ s.t. } & \sum_{j=1}^{k_v} \pi_{ij} = \mu_i \text{ for any } i\in [k], \quad \sum_{i=1}^{k} \pi_{ij} = \nu_j \text{ for any } j\in [k_v], \\
	                                                                                                                       & e^{-\varepsilon/2} m_j \leq \nu_j \leq e^{\varepsilon/2} m_j \text{ for any } j\in [k_v].
\end{align*}
where $\mathcal{X}=[k]$, $\mathcal{V}=[k_v]$, and $C_{ij}=d(x_i,v_j)^p$.
This LP can be reformulated as a minimum cost flow problem similarly to the optimal transport problem, and thus,
can be solved in $\tilde O(k^3)$ time~\citep{pele2009fast}. See Appendix~\ref{app:lp-projection} for details.
Thus, exactly solving the above problem is computationally expensive for large $k$ as in the case of optimal transport problems~\citep{cuturi2013sinkhorn}.

To overcome this issue, we propose to utilize entropic regularization to approximately compute $\mathcal{M}_{m, \varepsilon}[\mu]$ efficiently.
The entropically regularized optimal transport cost~\citep{cuturi2013sinkhorn} is defined as follows:
\begin{align*}
	\mathsf{OT}_\lambda(\nu, \mu) := \inf_{\pi \in \Pi(\nu, \mu)} \sum_{i,j} C_{ij}\pi_{ij} + \lambda \sum_{i,j} \pi_{ij}\log\pi_{ij},
\end{align*}
where $\lambda > 0$ is the regularization parameter.
The entropic regularization allows us to compute the optimal transport cost efficiently using the Sinkhorn algorithm~\citep{sinkhorn1967concerning,cuturi2013sinkhorn}, which is an iterative method for solving the regularized optimal transport problem.

On the other hand, our focus is computing the projection of a distribution $\mu$ onto the LDP polytope $Q_{m, \varepsilon}$ with respect to the Wasserstein distance.
If we replace the Wasserstein distance with the regularized transport cost, the approximate projection $\mathcal{M}_{m, \varepsilon}^{\lambda}[\mu]$ is formulated as:
\begin{align*}
	\mathcal{M}_{m, \varepsilon}^{\lambda}[\mu] := \arg\min_{\nu \in Q_{m, \varepsilon}} \mathsf{OT}_\lambda(\nu, \mu).
\end{align*}
Since $x\mapsto x^{1/p}$ is monotone on $\mathbb{R}_+$, this coincides with the exact Wasserstein projection when $\lambda=0$.
Note the range of $\mathcal{M}_{m, \varepsilon}^{\lambda}$ is still contained in $Q_{m, \varepsilon}$, and thus, $\mathcal{M}_{m, \varepsilon}^{\lambda}$ satisfies $\varepsilon$-LDP.

For a fixed base measure $m$, the next proposition quantifies the approximation error of the entropic objective relative to the exact Wasserstein projection.
\begin{proposition}[Entropic approximation error for fixed base measure]\label{prop:entropic-gap}
	Fix $m$ and $\varepsilon$ such that $Q_{m,\varepsilon}\neq\emptyset$.
	For $\mu\in\mathcal{P}([k])$, let
	\begin{align*}
		\nu^* & \in \arg\min_{\nu\in Q_{m,\varepsilon}} W_p(\nu,\mu),
		\qquad
		\nu^\lambda := \mathcal{M}_{m,\varepsilon}^{\lambda}[\mu].
	\end{align*}
	Then, we have
	\begin{align*}
		0 \le W_p(\nu^\lambda,\mu)-W_p(\nu^*,\mu) \le \bigl(2\lambda\log(k)\bigr)^{1/p}.
	\end{align*}
\end{proposition}
See Appendix~\ref{app:proof-entropic-gap} for the proof.
By choosing $\lambda$ sufficiently small, we can make the approximation error arbitrarily small.

Naively, the above optimization problem is a bi-level optimization problem, which can be computationally expensive to solve.
However, we can show that the above optimization problem can be reformulated as a KL projection onto
the intersection of two convex sets on the probability simplex and projection onto each set can be computed efficiently.
\begin{lemma}\label{lem:kl-projection}
	Let $\Delta_{k\times k_v} := \left\{\pi \in \R_+^{k\times k_v} : \sum_{i=1}^{k}\sum_{j=1}^{k_v} \pi_{ij} = 1\right\}$, and
	\begin{align*}
		R_\mu & := \{\pi \in \Delta_{k\times k_v} : \pi \1 = \mu\}, \quad
		R_{m,\varepsilon} := \{\pi \in \Delta_{k\times k_v} : \pi^\top \1 \in Q_{m,\varepsilon}\}.
	\end{align*}
	Assume $Q_{m,\varepsilon}$ is nonempty, equivalently
	$e^{-\varepsilon/2}\sum_j m_j \le 1 \le e^{\varepsilon/2}\sum_j m_j$.
	Then, we have
	\begin{align*}
		\pi_* & = \arg\min_{\pi \in R_\mu \cap R_{m, \varepsilon}} \mathrm{KL}(\pi \| K), \quad \mathcal{M}_{m, \varepsilon}^{\lambda}[\mu] = \pi_*^\top \1,
	\end{align*}
	where $K$ is the Gibbs kernel (depending on $\lambda$) defined as
	$
		K_{i, j} = \frac{1}{Z_\lambda} \exp\left(-\frac{C_{i, j}}{\lambda}\right),
	$
	and $Z_\lambda$ is the normalization constant.
	Furthermore, projection onto $R_\mu$ and $R_{m, \varepsilon}$ can be computed as follows:
	\begin{align*}
		\proj^{\kl}_{R_\mu}(\pi)_{ij} & = \frac{\mu_i}{\sum_j \pi_{ij}} \pi_{ij}, \quad \proj^{\kl}_{R_{m,\varepsilon}}(\pi)_{ij} = \frac{q_j^*}{\sum_i \pi_{ij}}\pi_{ij},
	\end{align*}
	where $s_j := \sum_i \pi_{ij}$ and
	$
		q^* = \proj^{\kl}_{Q_{m,\varepsilon}}(s) = \arg\min_{q \in Q_{m,\varepsilon}} \mathrm{KL}(q \| s).
	$
	Moreover, $q_j^* = \min\{\max\{e^{\theta}s_j, e^{-\varepsilon/2}m_j\}, e^{\varepsilon/2}m_j\}$,
	where $\theta$ is a scalar chosen so that $\sum_j q_j^* = 1$.
\end{lemma}
See Appendix~\ref{app:proof-kl-projection} for the proof. The projection $q^*$ onto $Q_{m,\varepsilon}$ can be computed efficiently using scalar bisection since $q_j^*$ is a monotone function of $\theta$ as discussed in~\citet{park2024exactly}.
For simplicity, we assume that $q^*$ can be computed exactly in the following analysis.

The above problem is a special case of Bregman projection onto the intersection of convex sets, and can be solved using the Dykstra's algorithm~\citep{benamou2015iterative}
based on alternating projections onto $R_\mu$ and $R_{m,\varepsilon}$.
In general, Dykstra's algorithm requires additional correction terms to ensure convergence to the optimal solution
and convergence rate cannot be easily characterized.
However, utilizing the special structure of the problem, we can prove that a vanilla iterative Bregman projection algorithm
without Dykstra's correction is sufficient to find the optimal solution, which significantly reduces the computational cost.
Furthermore, we can show that the algorithm converges linearly to the optimal solution, which is a stronger convergence guarantee than the general case.
We provide the details of the algorithm in Algorithm~\ref{alg:projection}.

\begin{algorithm}[ht]
	\caption{Projection onto LDP polytope with entropic regularization}
	\label{alg:projection}
	\begin{algorithmic}[1]
		\STATE \textbf{Input:} Target distribution $\mu$, base measure $m$, privacy parameter $\varepsilon$, regularization parameter $\lambda$, cost matrix $C$.
		\STATE \textbf{Output:} Projected distribution $\mathcal{M}_{m, \varepsilon}^{\lambda}[\mu]$.
		\STATE Initialize $v^{(0)} \in \mathbb{R}_{++}^{k_v}$.
		\STATE Compute the Gibbs kernel $K$ as $[K]_{i, j} = \frac{1}{Z_\lambda} \exp\left(-\frac{C_{i, j}}{\lambda}\right)$.
		\FOR{$t = 0, 1, 2, \dots$ until convergence}
		\STATE Update $u^{(t)} = \frac{\mu}{K v^{(t)}}$.
		\STATE Set $s^{(t)} = K^\top u^{(t)}$.
		\STATE Compute $q^{(t)} = \operatorname{Proj}^{\mathrm{KL}}_{Q_{m,\varepsilon}}(s^{(t)})$ via scalar bisection.
		\STATE Update $v^{(t+1)} = \frac{q^{(t)}}{s^{(t)}}$.
		\ENDFOR
		\STATE Return $\mathcal{M}_{m, \varepsilon}^{\lambda}[\mu] = q^{(t)}$.
	\end{algorithmic}
\end{algorithm}

To analyze the convergence of the above algorithm,
we define the composite mapping $T : \mathbb{R}^{k_v}_{++} \to \mathbb{R}^{k_v}_{++}$ by
$
	T(v) := \frac{\Pi\bigl(K^\top (\mu/(K v))\bigr)}{K^\top (\mu/(K v))},
$
where $\Pi(s) := \proj^{\kl}_{Q_{m,\varepsilon}}(s)$. Note that $v^{(t)} = T^t(v^{(0)})$ for all $t \ge 0$.
We equip $\mathbb{R}^{k_v}_{++}$ with the Hilbert metric defined by
$
	d_H(x,y) = \log \frac{\max_i (x_i/y_i)}{\min_i (x_i/y_i)}.
$

Then, we have the following convergence guarantee:
\begin{theorem}[Linear convergence of Algorithm~\ref{alg:projection}]\label{thm:convergence}
	Assume that $\mu_i > 0$ for all $i \in [k]$, $m_j > 0$ for all $j \in [k_v]$, and $K_{i,j} > 0$ for all $i \in [k], j \in [k_v]$.
	(If some $\mu_i$ or $m_j$ vanish, one may restrict $K$ to the rows and columns in $\operatorname{supp}(\mu) \times \operatorname{supp}(m)$.)
	Let $\tau(\cdot)$ denote the Birkhoff contraction coefficient~\citep{birkhoff1957extensions}
	\begin{align*}
		\tau(A) = \sup_{x,y \in \mathbb{R}^{k_v}_{++}} \frac{d_H(Ax,Ay)}{d_H(x,y)},
	\end{align*}
	which is strictly less than 1 for any positive matrix $A$, and $c := \tau(K^\top)\tau(K) \in (0,1)$.
	Then for any initial point $v^{(0)} \in \mathbb{R}^{k_v}_{++}$, the sequence $\{v^{(t)}\}_{t=0}^\infty$ generated
	by Algorithm~\ref{alg:projection} satisfies
	\begin{align*}
		d_H(v^{(t)},v^*) \le c^t \, d_H(v^{(0)},v^*),
	\end{align*}
	where $v^*$ is a fixed point of $T$ (unique up to positive scaling).
	Moreover, defining
	\begin{align*}
		s^{(t)} & := K^\top\!\left(\frac{\mu}{Kv^{(t)}}\right),\quad q^{(t)} := \Pi(s^{(t)}),\quad s^* := K^\top\!\left(\frac{\mu}{Kv^*}\right),\quad q^* := \Pi(s^*),
	\end{align*}
	we have $v^* = q^*/s^*$ (componentwise) and
	\begin{align*}
		\mathrm{KL}(q^{(t)}\|q^*)
		 & \le 2c^{t+1} d_H(v^{(0)},v^*).
	\end{align*}
\end{theorem}
See Appendix~\ref{app:proof-convergence} for the proof.
Theorem~\ref{thm:convergence} guarantees linear convergence of both the iterates $\{v^{(t)}\}$ in Hilbert metric and the output distributions $\{q^{(t)}\}$ in KL divergence.
These results ensure the efficiency of Algorithm~\ref{alg:projection}.

\section{Numerical Experiments}
In this section, we validate our proposed methods through numerical experiments on synthetic data and the MovieLens dataset.
We also provide additional experiments on geographical datasets in Appendix~\ref{app:additional-experiments}.
\begin{figure}[t]
	\begin{minipage}{0.66\textwidth}
		\includegraphics[width=0.495\textwidth]{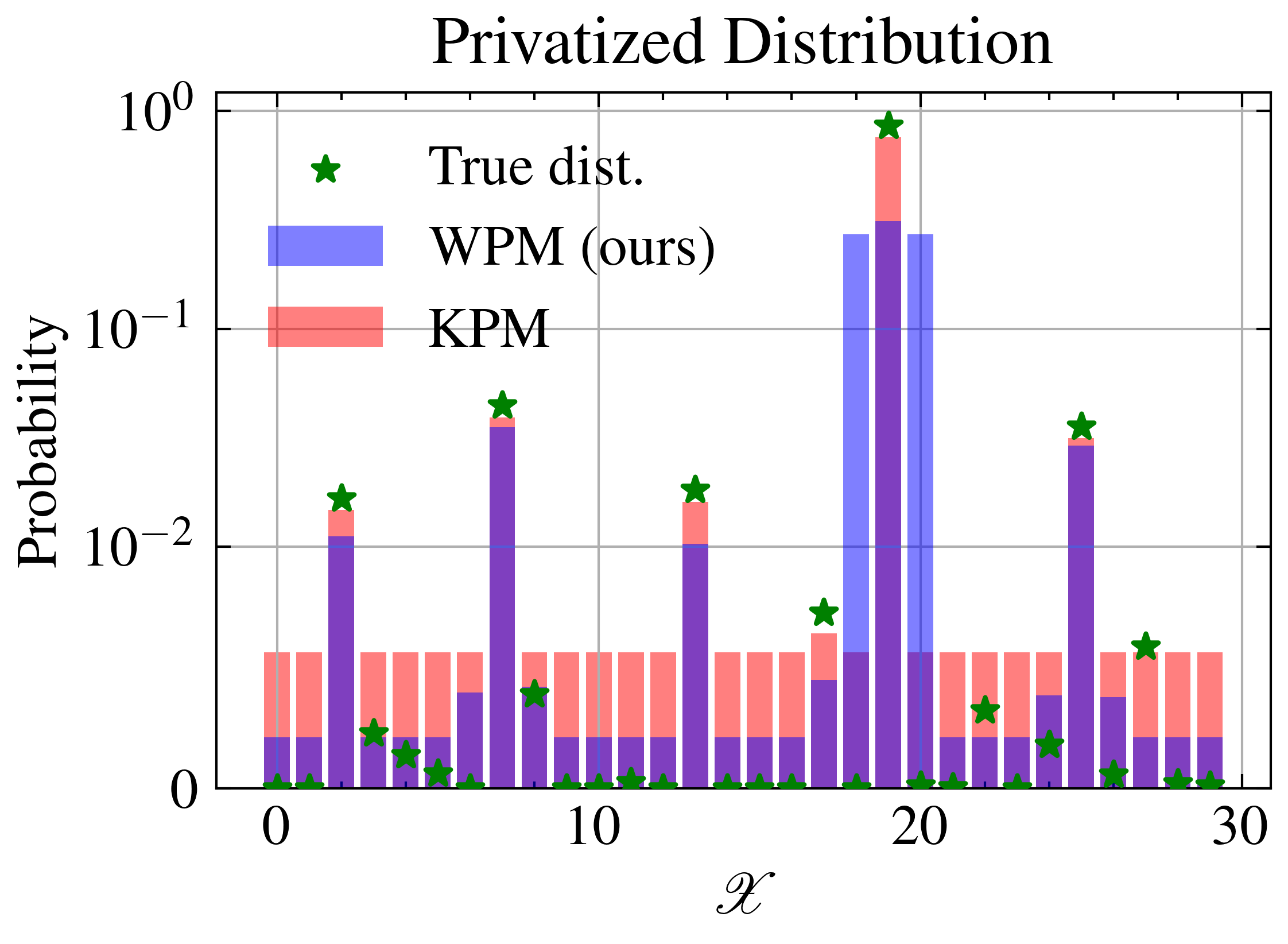}
		\includegraphics[width=0.495\textwidth]{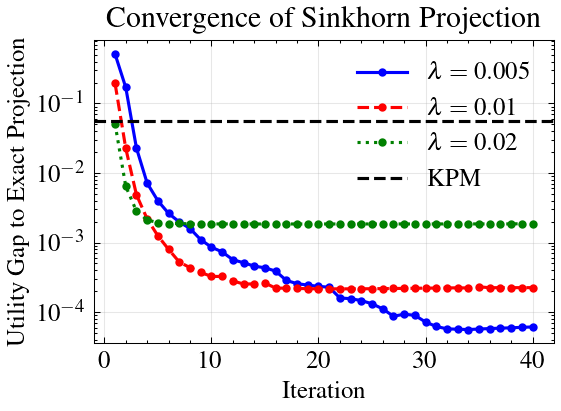}
		\caption{Results for synthetic experiments.
			Left: Transformed distribution with WPM and KPM.
			Right: Convergence speed of Algorithm~\ref{alg:projection}.}
		\label{fig:synthetic}
	\end{minipage}
	\begin{minipage}{0.33\textwidth}
		\includegraphics[width=0.99\textwidth]{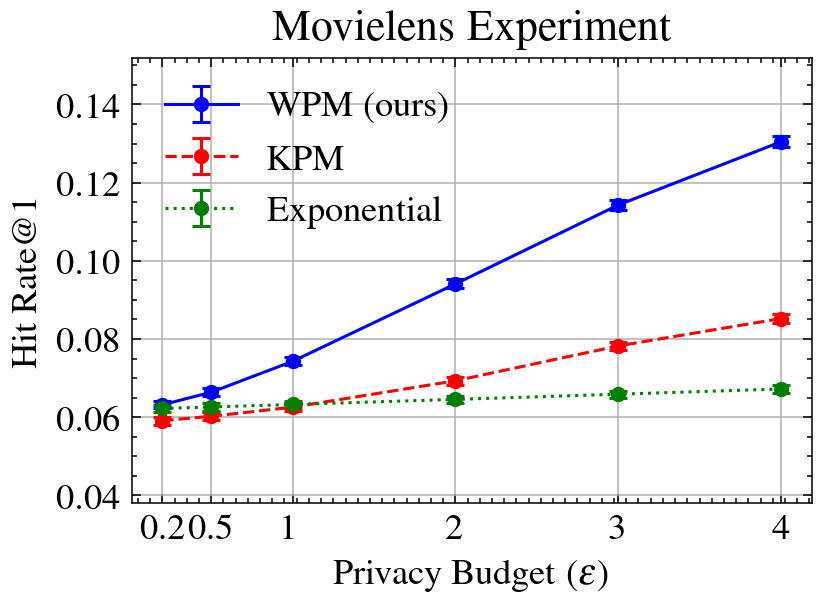}
		\caption{Results for MovieLens experiments. Our WPM outperforms the baselines.}
		\label{fig:movielens}
	\end{minipage}
\end{figure}

\subsection{Synthetic Data}
To see the difference of privatized distributions between the WPM (ours) and the KPM baseline,
we consider a simple synthetic example where $\mathcal{X}=\mathcal{V}= \{0, \dots, k-1\}$ is a discrete set of $k$ elements on a ring,
and the cost matrix $C$ is defined as $d(i, j) = \min\{|i-j|, k - |i-j|\}$.
Unless otherwise specified, we set $k = 30$, $\varepsilon = 5$, and $p = 2$.
We construct an input distribution $\mu$ by sampling from a Dirichlet distribution with concentration parameter $0.1$ and normalizing the sample to sum to 1.

\paragraph{WPM captures the underlying geometry}
Fig.~\ref{fig:synthetic} (Left) shows an example of the projected distributions of $\mu$ with WPM and KPM.
While KPM produces a uniformly spread distribution,
WPM captures the geometric structure and assigns high probability to regions near the true high-probability areas,
and low probability to regions that are not.

\paragraph{Sinkhorn projection converges rapidly}
Fig.~\ref{fig:synthetic} (Right) shows the convergence of Algorithm~\ref{alg:projection} in terms of the Wasserstein distance between the input and the output distributions.
We see that our proposed algorithm converges rapidly near the exact projection.
For large $\lambda$, the algorithm converges faster but the approximation error is larger, while for small $\lambda$, the algorithm converges slower but the approximation error is smaller, which is consistent with Proposition~\ref{prop:entropic-gap}.

\subsection{MovieLens Dataset}
Following the experimental setup of~\citet{zamanlooy2025locally},
we consider a movie recommendation scenario, where the goal is to privately recommend the next movie based on the user's past movie ratings.
While \citet{zamanlooy2025locally} consider the problem of genre recommendation,
we consider the more practical problem of movie recommendation, where the support (i.e., the set of candidate movies for recommendation)
is essentially different from the support of the input distribution (i.e., movies that the user has rated in the past).
In such scenarios, $f$-divergences are not suitable for measuring the distance between the input and output distributions,
while Wasserstein distance can still capture the distance between the distributions by utilizing the underlying geometry of the movie space.
As a baseline, we consider the $f$-divergence based method which first samples from the privatized genre distribution and then selects a movie uniformly at random from the selected genre,
which is a natural extension of the method proposed by~\citet{zamanlooy2025locally}.
In addition, we also consider the Exponential Mechanism~\citep{mcsherry2007mechanism} with a utility function defined as the negative expected distance $u(\mu, j) = -E_{i \sim \mu}[d(i, j)]$.

We use the MovieLens100k dataset~\citep{harper2015movielens} and divide movies into train and test sets.
We construct the input distribution so that it is proportional to the ratings of the movies the user has rated in the train set, and
sample a movie from the output distribution of the mechanism over the candidate movies in the test set.
As an underlying metric, we compute embeddings of the movies via matrix factorization, capturing user co-consumption patterns, and use cosine distance between the embeddings.
In this experiment, we regard the base measure as a hyperparameter and select the best base measure for each method from a set of uniform measures based on the validation performance.
We evaluate each method by computing Hit Rate@1, the probability that the output movie is included in the user’s test data with high ratings ($\ge 4$).
The error bars in the figures represent the standard error.

\paragraph{WPM works under support mismatch}
Fig.~\ref{fig:movielens} shows the performance of WPM and KPM in terms of Hit Rate@1.
We see that WPM significantly outperforms KPM across all values of $\varepsilon$, and the performance gap becomes larger as $\varepsilon$ increases.
This is because the WPM can handle the movie distribution directly even under support mismatch by utilizing the underlying geometry of the movie space,
while the KPM baseline method suffers from the support mismatch issue and thus, its performance degrades significantly.
Exponential Mechanism also performs worse, which may be due to the fact that it only depends on the mean of the input distribution, while WPM can utilize the full distributional information.



\section{Conclusion}\label{sec:conclusion}
In this paper, we proposed a novel framework for differentially private sampling from distributions based on Wasserstein geometry.
We provide the minimax optimal mechanism in terms of worst-case utility based on Wasserstein projection and we show that the optimal mechanism can be obtained by solving a convex optimization problem.
We also proposed a practical algorithm for computing approximate projections using entropic regularization and provided convergence guarantees.
Finally, we demonstrated the effectiveness of our proposed method through numerical experiments on synthetic and real-world datasets.

\paragraph{Limitations and Future Work}
Our main contribution lies in the theoretical framework and analysis of optimal private sampling mechanisms based on Wasserstein geometry.
Since our framework is quite general and can be applied to various settings, extensive empirical evaluation in specific applications is left for future work.
In addition, our focus in this paper is on the local model of differential privacy, and extending our framework to the centralized model of differential privacy is an interesting direction for future research.

\clearpage

\bibliographystyle{plainnat}
\bibliography{main}


\appendix

\section{Detailed Discussion of Related Work} \label{app:related-work}
In this section, we show that our framework can be viewed as a generalization of the classical estimation problem in DP.
Let us consider the matched-support setting $\mathcal X=\mathcal V$ where the user has a single data point $x \in \mathcal{X}$,
and the goal is to design a randomized mechanism $\mathcal{L}:\mathcal{X} \to \mathcal{X}$
that outputs a privatized version of $x$ while satisfying $\varepsilon$-LDP.
The utility of the mechanism is typically measured by the mean squared error $\mathbb{E}[\|x - \mathcal{L}(x)\|_2^2]$.

On the other hand, in our framework, let us consider the case where
the input distribution $\mu$ is a Dirac measure at $x$, i.e., $\mu = \delta_x$.
If we set $d(x, y) = \|x - y\|_2$ and $p=2$,
then the Wasserstein distance between $\mathcal{M}[\delta_x]$ and $\delta_x$ is given by
\begin{align*}
	W_2^2(\mathcal{M}[\delta_x], \delta_x) = \int_{\mathcal{X}} \|x - y\|_2^2 \, d\mathcal{M}[\delta_x](y) = \mathbb{E}[\|x - \mathcal{L}(x)\|_2^2].
\end{align*}
This matches the mean squared error of the mechanism $\mathcal{L}$ in the classical estimation problem.

Therefore, classical estimation under LDP can be seen as a special case of our framework where the input distribution is a Dirac measure.
Note that estimation under LDP typically focuses on estimating a mean across users, and thus it poses an unbiasedness constraint on the mechanism.
In such a case, the problem is different from our framework, which does not require unbiasedness.

\section{Auxiliary Results}
\begin{lemma}\label{thm:polytope}
	A mechanism $\mathcal{M}$ satisfies $\varepsilon$-LDP
	if and only if there exists a base measure $m$ such that for all $\mu \in \mathcal{P}(\mathcal{X})$, $\mathcal{M}[\mu] \in Q_{m, \varepsilon}$.
\end{lemma}

\begin{proof}
	\emph{If Part.}
	Suppose there exists a base measure $m$ such that for all $\mu \in \mathcal{P}(\mathcal{X})$, $\mathcal{M}[\mu] \in Q_{m, \varepsilon}$.
	Then, we have for all $\mu, \mu' \in \mathcal{P}(\mathcal{X})$ and for all measurable subsets $A \subseteq \mathcal{V}$:
	\begin{align*}
		\mathcal{M}[\mu](A)  & \leq e^{\varepsilon / 2} m(A)  \\
		\mathcal{M}[\mu'](A) & \geq e^{-\varepsilon / 2} m(A)
	\end{align*}
	and thus,
	\begin{align*}
		\mathcal{M}[\mu](A) \leq e^{\varepsilon} \mathcal{M}[\mu'](A).
	\end{align*}
	This shows that $\mathcal{M}$ satisfies $\varepsilon$-LDP.

	\emph{Only if Part.}
	Suppose $\mathcal{M}$ satisfies $\varepsilon$-LDP.
	Fix an arbitrary $\mu_0 \in \mathcal{P}(\mathcal{X})$ and define a reference measure $\lambda = \mathcal{M}[\mu_0]$.
	For any $\mu \in \mathcal{P}(\mathcal{X})$ and measurable $A \subseteq \mathcal{V}$,
	$\varepsilon$-LDP gives
	\begin{align*}
		\mathcal{M}[\mu](A) \leq e^{\varepsilon} \lambda(A).
	\end{align*}
	Hence $\mathcal{M}[\mu] \ll \lambda$, so by the Radon-Nikodym theorem
	$f_\mu := \frac{d\mathcal{M}[\mu]}{d\lambda}$ exists for each $\mu$.

	Moreover, for any $\mu,\mu' \in \mathcal{P}(\mathcal{X})$, $\varepsilon$-LDP implies
	$\mathcal{M}[\mu](A) \le e^{\varepsilon}\mathcal{M}[\mu'](A)$ for all measurable $A$.
	Taking its Radon-Nikodym derivative with respect to $\lambda$ yields
	\begin{align}
		f_\mu(x) \le e^{\varepsilon}f_{\mu'}(x) \quad \text{for } \lambda\text{-almost every }x. \label{eq:pointwise-inequality}
	\end{align}

	Define
	\begin{align*}
		\overline{f} := \operatorname{ess\,sup}_{\mu \in \mathcal{P}(\mathcal{X})} f_\mu,
		\qquad
		\underline{f} := \operatorname{ess\,inf}_{\mu \in \mathcal{P}(\mathcal{X})} f_\mu.
	\end{align*}
	Since $\lambda$ is a probability measure, especially a $\sigma$-finite measure,
	there is a countable set $\mathcal{C} \subset \mathcal{P}(\mathcal{X})$ such that
	\begin{align*}
		\overline{f}(x) = \sup_{\mu \in \mathcal{C}} f_\mu(x), \quad
		\underline{f}(x) = \inf_{\mu \in \mathcal{C}} f_\mu(x)
	\end{align*}
	for $\lambda$-almost every $x$.
	From Eq.~\eqref{eq:pointwise-inequality}, we have
	\begin{align*}
		\overline{f}(x) \leq e^{\varepsilon} \underline{f}(x)
	\end{align*}
	for $\lambda$-almost every $x$.

	We now construct the base measure $m$ by defining its density with respect to $\lambda$ as $dm(x) = \sqrt{\overline{f}(x)\underline{f}(x)} \, d\lambda(x)$.
	For any $\mu \in \mathcal{P}(\mathcal{X})$, we have $\underline{f}(x) \leq f_\mu(x) \leq \overline{f}(x)$ $\lambda$-almost everywhere.
	Thus, $\lambda$-almost everywhere,
	\begin{align*}
		f_\mu(x) & \leq \sqrt{\overline{f}(x) \cdot \overline{f}(x)} \leq \sqrt{e^{\varepsilon} \underline{f}(x) \cdot \overline{f}(x)} = e^{\varepsilon / 2} \sqrt{\overline{f}(x)\underline{f}(x)}      \\
		f_\mu(x) & \geq \sqrt{\underline{f}(x) \cdot \underline{f}(x)} \geq \sqrt{e^{-\varepsilon} \overline{f}(x) \cdot \underline{f}(x)} = e^{-\varepsilon / 2} \sqrt{\overline{f}(x)\underline{f}(x)}.
	\end{align*}

	Integrating these bounds over any measurable subset $A \subseteq \mathcal{V}$ with respect to $\lambda$, we obtain:
	\begin{align*}
		e^{-\varepsilon / 2} m(A) \leq \mathcal{M}[\mu](A) \leq e^{\varepsilon / 2} m(A).
	\end{align*}
	Therefore, $\mathcal{M}[\mu] \in Q_{m, \varepsilon}$ for all $\mu \in \mathcal{P}(\mathcal{X})$.
\end{proof}

\begin{proposition}\label{prop:dirac}
	For $x \in \mathcal{X}$, let $\delta_x$ be the Dirac measure at $x$.
	Then, if $\mathcal{M}$ is a Wasserstein projection, we have
	\begin{align*}
		\sup_{\mu \in \mathcal{P}(\mathcal{X})} W_p(\mathcal{M}[\mu], \mu) = \sup_{x \in \mathcal{X}} W_p(\mathcal{M}[\delta_x], \delta_x).
	\end{align*}
\end{proposition}
\begin{proof}
	Let $Q := Q_{m,\varepsilon}$ and assume that
	\begin{align*}
		\mathcal M[\mu] \in \arg\min_{\nu\in Q} W_p(\nu,\mu)
	\end{align*}
	for every $\mu\in\mathcal P(\mathcal X)$.

	The inequality
	\begin{align*}
		\sup_{x\in\mathcal X} W_p(\mathcal M[\delta_x],\delta_x)
		\leq
		\sup_{\mu\in\mathcal P(\mathcal X)} W_p(\mathcal M[\mu],\mu)
	\end{align*}
	is immediate since $\delta_x \in \mathcal P(\mathcal X)$.

	For the reverse inequality, fix $\mu\in\mathcal P(\mathcal X)$ and define
	\begin{align*}
		\nu_x := \mathcal M[\delta_x], \qquad
		\bar\nu := \int_{\mathcal X} \nu_x \, d\mu(x).
	\end{align*}
	Since $Q$ is convex, $\bar\nu \in Q$. By the projection property,
	\begin{align*}
		W_p^p(\mathcal M[\mu],\mu) \leq W_p^p(\bar\nu,\mu).
	\end{align*}

	Let
	\begin{align*}
		\pi_x(dy,dz)=\nu_x(dy)\delta_x(dz)
	\end{align*}
	and define
	\begin{align*}
		\pi := \int_{\mathcal X} \pi_x\, d\mu(x).
	\end{align*}
	Then $\pi$ is a coupling of $\bar\nu$ and $\mu$. Hence
	\begin{align*}
		W_p^p(\bar\nu,\mu) & \leq \int_{\mathcal X\times\mathcal X} d(y,z)^p\, d\pi(y,z)                     \\
		                   & = \int_{\mathcal X}\!\left(\int_{\mathcal V} d(y,x)^p\, d\nu_x(y)\right)d\mu(x) \\
		                   & = \int_{\mathcal X} W_p^p(\nu_x,\delta_x)\, d\mu(x)                             \\
		                   & = \int_{\mathcal X} W_p^p(\mathcal M[\delta_x],\delta_x)\, d\mu(x)              \\
		                   & \leq \sup_{x\in\mathcal X} W_p^p(\mathcal M[\delta_x],\delta_x).
	\end{align*}
	Therefore
	\begin{align*}
		W_p(\mathcal M[\mu],\mu) \leq \sup_{x\in\mathcal X} W_p(\mathcal M[\delta_x],\delta_x).
	\end{align*}
	Taking supremum over $\mu$ completes the proof.
\end{proof}

\section{Proof of Proposition~\ref{prop:projection-ldp}} \label{app:proof-ldp}
\begin{proof}
	Fix $\mu,\mu'\in\mathcal P(\mathcal X)$ and $A\subseteq\mathcal V$ measurable.
	By the definition of $Q_{m,\varepsilon}$,
	\begin{align*}
		\mathcal M_{m,\varepsilon}[\mu](A)  & \le \beta m(A)  \\
		\mathcal M_{m,\varepsilon}[\mu'](A) & \ge \alpha m(A)
	\end{align*}
	where $\alpha:=e^{-\varepsilon/2}$ and $\beta:=e^{\varepsilon/2}$.
	Hence
	\begin{align*}
		\mathcal M_{m,\varepsilon}[\mu](A) \le e^{\varepsilon}\mathcal M_{m,\varepsilon}[\mu'](A).
	\end{align*}
	This proves that $\mathcal M_{m,\varepsilon}$ satisfies $\varepsilon$-LDP.
\end{proof}

\section{Proof of Proposition~\ref{thm:projection}} \label{app:proof-projection}
\begin{proof}
	For any $\varepsilon$-LDP mechanism $\mathcal M$, by Lemma~\ref{thm:polytope}, there exists a base measure $m$ such that for all $\mu$,
	\begin{align*}
		\mathcal M[\mu] \in Q_{m,\varepsilon}.
	\end{align*}
	By the definition of $\mathcal M_{m,\varepsilon}$, $\mathcal M_{m,\varepsilon}[\mu]$ is the projection of $\mu$ onto $Q_{m,\varepsilon}$ with respect to $W_p$.
	From the property of projection, for all $\nu \in Q_{m,\varepsilon}$,
	\begin{align*}
		W_p(\mathcal M_{m,\varepsilon}[\mu],\mu) \le W_p(\nu,\mu).
	\end{align*}
	Taking $\nu = \mathcal M[\mu] \in Q_{m,\varepsilon}$ gives the desired result.
\end{proof}

\section{Proof of Theorem~\ref{thm:minimax}} \label{app:proof-minimax}
\begin{proof}
	By Theorem~\ref{thm:projection}, for any $\varepsilon$-LDP mechanism $\mathcal M$, there exists a base measure $m$ such that for all $\mu$,
	\begin{align*}
		W_p(\mathcal M[\mu],\mu) \ge W_p(\mathcal M_{m,\varepsilon}[\mu],\mu).
	\end{align*}
	Hence
	\begin{align*}
		\inf_{\varepsilon\text{-LDP }\mathcal M}\sup_{\mu} W_p(\mathcal M[\mu],\mu)
		 & \ge \inf_{m}\sup_{\mu} W_p(\mathcal M_{m,\varepsilon}[\mu],\mu) \\
		 & \ge \inf_{m}\max_i \phi_i(m)^{1/p}.
	\end{align*}
	The reverse inequality is immediate since $\mathcal M_{m,\varepsilon}$ is $\varepsilon$-LDP for every $m$.
	Thus, it suffices to show that the infimum over $m$ is attained by some $m^*$.

	By Step 4 in the proof of Theorem~\ref{thm:base-measure-minimax}, $m\mapsto\max_i \phi_i(m)$ is a convex function on $\R_+^{k_v}$.
	Moreover, the constraint set $\{m\in\R_+^{k_v}: 1/\beta \le \sum_j m_j \le 1/\alpha\}$ is a compact convex set.
	Therefore, the infimum is attained by some $m^*$.
\end{proof}

\section{Proof of Proposition~\ref{thm:base-measure-minimax}} \label{app:proof-base-measure-minimax}

\begin{proof}
	Throughout, let $\mathcal X=[k]$, $\mathcal V=[k_v]$, $\alpha:=e^{-\varepsilon/2}$, and $\beta:=e^{\varepsilon/2}$.

	\paragraph{Step 1: Coordinate form of $Q_{m,\varepsilon}$.}
	In the discrete case, for any nonnegative vector $m\in\R_+^{k_v}$ and any probability vector $\nu\in\mathcal P([k_v])$,
	\begin{align*}
		\nu\in Q_{m,\varepsilon}
		\iff
		\alpha m_i\le \nu_i\le \beta m_i \quad (\forall i\in[k_v]).
	\end{align*}
	The ``only if'' direction follows by taking $A=\{i\}$ in the definition of $Q_{m,\varepsilon}$.
	The converse follows by summing coordinate-wise inequalities over any $A\subseteq[k_v]$.

	\paragraph{Step 2: Reduction to Dirac inputs.}
	By Proposition~\ref{prop:dirac}, for projection mechanisms,
	\begin{align*}
		\sup_{\mu\in\mathcal P(\mathcal X)} W_p\bigl(\mathcal M_{m,\varepsilon}[\mu],\mu\bigr)
		=
		\max_{i\in[k]} W_p\bigl(\mathcal M_{m,\varepsilon}[\delta_i],\delta_i\bigr).
	\end{align*}
	Since $x\mapsto x^p$ is strictly increasing on $\R_+$,
	minimizing the left-hand side over $m$ is equivalent to minimizing
	\begin{align*}
		\max_{i\in[k]} W_p^p\bigl(\mathcal M_{m,\varepsilon}[\delta_i],\delta_i\bigr).
	\end{align*}

	\paragraph{Step 3: Inner optimization for each $i$.}
	Fix $i\in[k]$. For any $\nu\in\mathcal P([k_v])$, since one marginal is $\delta_i$, the only coupling of $(\nu,\delta_i)$ is
	\begin{align*}
		\pi(j,\ell)=\nu_j\mathbf 1_{\{\ell=i\}}.
	\end{align*}
	Therefore,
	\begin{align*}
		W_p^p(\nu,\delta_i)
		= \sum_{j=1}^{k_v} d(v_j,x_i)^p\,\nu_j
		= \sum_{j=1}^{k_v} C_{ij}\,\nu_j.
	\end{align*}
	Hence
	\begin{align*}
		W_p^p\bigl(\mathcal M_{m,\varepsilon}[\delta_i],\delta_i\bigr)
		=
		\min_{q\in\R^{k_v}}
		\left\{\sum_{j=1}^{k_v} C_{ij}q_j:
		\sum_{j=1}^{k_v} q_j=1,\ q_j\ge 0,\ \alpha m_j\le q_j\le \beta m_j\ (\forall j)\right\}.
	\end{align*}
	The above problem is a fractional knapsack problem: for each $i\in[k]$, sort indices by increasing $C_{ij}$ and greedily fill capacities $(\beta-\alpha)m_j$.
	Therefore, $\phi_i(m)$ can be expressed in the following closed form:
	\begin{align*}
		\phi_i(m)
		 & = \alpha \sum_{j: C_{ij} > \tau_i(m)} C_{ij} m_j + \beta \sum_{j: C_{ij} < \tau_i(m)} C_{ij} m_j               \\
		 & \quad + \tau_i(m)\left(1-\alpha\sum_{j: C_{ij} > \tau_i(m)} m_j - \beta\sum_{j: C_{ij} < \tau_i(m)} m_j\right) \\
		 & = \tau_i(m) + \sum_{j=1}^{k_v} (\alpha(C_{ij}-\tau_i(m))_+ - \beta(\tau_i(m)-C_{ij})_+) m_j,
	\end{align*}
	where $\tau_i(m)$ is defined as $\tau_i(m) := \min\{t: \sum_{j=1}^{k_v} (\beta m_j \mathbf{1}_{\{C_{ij} \leq t\}} + \alpha m_j \mathbf{1}_{\{C_{ij}>t\}}) \ge 1\}$.

	Combining Steps 2 and 3,
	\begin{align*}
		\inf_{m\in\R_+^{k_v}}\sup_{\mu\in\mathcal P(\mathcal X)} W_p\bigl(\mathcal M_{m,\varepsilon}[\mu],\mu\bigr)
		=
		\left(\inf_{m\in\R_+^{k_v}} \max_{i\in[k]} \phi_i(m)\right)^{1/p}.
	\end{align*}
	Therefore, minimizing worst-case utility is equivalent to
	\begin{align*}
		\min_{m\in\R_+^{k_v}} \max_{i\in[k]} \phi_i(m).
	\end{align*}

	\paragraph{Step 4: Convexity in the base measure.}
	Define
	\begin{align*}
		\mathcal Q(m)
		:=
		\left\{q\in\R_+^{k_v}:
		\sum_{j=1}^{k_v} q_j=1,
		\alpha m_j\le q_j\le \beta m_j\ (\forall j)\right\}.
	\end{align*}
	For $m,m'\in\R_+^{k_v}$ and $\theta\in[0,1]$, let $m_\theta:=\theta m+(1-\theta)m'$. If either
	$\mathcal Q(m)$ or $\mathcal Q(m')$ is empty, convexity of the extended-value function
	$\phi_i$ is immediate. Otherwise, for any $q\in\mathcal Q(m)$ and $q'\in\mathcal Q(m')$,
	$q_\theta:=\theta q+(1-\theta)q'$ satisfies
	\begin{align*}
		\sum_j q_{\theta,j}=1,
		\qquad
		\alpha m_{\theta,j}\le q_{\theta,j}\le \beta m_{\theta,j}\ (\forall j),
	\end{align*}
	so $q_\theta\in\mathcal Q(m_\theta)$. Hence
	\begin{align*}
		\phi_i(m_\theta)
		 & \le \sum_{j=1}^{k_v} C_{ij}q_{\theta,j}
		= \theta\sum_{j=1}^{k_v} C_{ij}q_j + (1-\theta)\sum_{j=1}^{k_v} C_{ij}q'_j.
	\end{align*}
	Taking infimum over $q\in\mathcal Q(m)$ and $q'\in\mathcal Q(m')$ gives
	\begin{align*}
		\phi_i\bigl(\theta m+(1-\theta)m'\bigr)
		\le
		\theta\phi_i(m)+(1-\theta)\phi_i(m');
	\end{align*}
	Thus each $\phi_i$ is convex (as an extended-value function on $\R_+^{k_v}$), and therefore
	$m\mapsto\max_i\phi_i(m)$ is convex.
	This completes the proof.
\end{proof}

\section{Algorithm details for base-measure optimization}\label{app:mirror-descent-base-measure}
Define
\begin{align*}
	f(m):=\max_{i\in[k]}\phi_i(m),
	\qquad
	D_p:=\max_{i\in[k],\,j\in[k_v]} C_{ij}.
\end{align*}
Let
\begin{align*}
	\mathcal M
	:=
	\left\{m\in\R_+^{k_v}:\frac{1}{\beta}\le \sum_{j=1}^{k_v} m_j\le \frac{1}{\alpha}\right\}.
\end{align*}

By Step 3 in the proof of Theorem~\ref{thm:base-measure-minimax}, $\phi_i(m)$ has an equivalent LP form. For fixed $i$ and $m$, write $q_j=\alpha m_j+r_j$. Then
\begin{align*}
	\phi_i(m)
	=
	\alpha\sum_{j=1}^{k_v} C_{ij}m_j
	+
	\min_r\left\{\sum_{j=1}^{k_v} C_{ij}r_j:
	0\le r_j\le (\beta-\alpha)m_j\ (\forall j),
	\sum_{j=1}^{k_v} r_j = 1-\alpha\sum_{j=1}^{k_v} m_j\right\}.
\end{align*}
Hence feasibility is equivalent to
\begin{align*}
	0\le 1-\alpha\sum_j m_j \le (\beta-\alpha)\sum_j m_j,
\end{align*}
i.e.,
\begin{align*}
	\frac{1}{\beta}\le \sum_{j=1}^{k_v} m_j\le \frac{1}{\alpha}.
\end{align*}

Let $\tau_i(m)$ denote the threshold cost of the greedy solution. To obtain a subgradient with
respect to $m$, we derive the LP dual.
Returning to the original primal form (Step~3), introduce dual variables
$\eta\in\R$ for $\sum_j q_j=1$,
$s_j\ge 0$ for $q_j\ge\alpha m_j$, and
$t_j\ge 0$ for $q_j\le\beta m_j$.
The Lagrangian is
\begin{align*}
	\mathcal L
	 & =
	\sum_j C_{ij}q_j
	+\eta\bigl(1-\textstyle\sum_j q_j\bigr)
	-\sum_j s_j(q_j-\alpha m_j)
	-\sum_j t_j(\beta m_j-q_j) \\
	 & =
	\eta+\sum_j(\alpha s_j-\beta t_j)m_j
	+\sum_j\underbrace{(C_{ij}-\eta-s_j+t_j)}_{=0\text{ at optimum}}q_j.
\end{align*}
Taking $\inf_q$ is finite iff $s_j-t_j=C_{ij}-\eta$ for all $j$.
Hence the dual is
\begin{align*}
	\phi_i(m)
	=
	\max_{\eta\in\R,\ s,t\in\R_+^{k_v}}
	\Bigl\{
	\eta + \sum_{j=1}^{k_v} (\alpha s_j-\beta t_j)m_j:
	s_j-t_j=C_{ij}-\eta\ (\forall j)
	\Bigr\}.
\end{align*}
Hence, for any dual optimizer $(\eta^*,s^*,t^*)$, Danskin's theorem gives
\begin{align*}
	g_j^{(i)}(m):=\alpha s_j^*-\beta t_j^*,\qquad j\in[k_v],
\end{align*}
and $g^{(i)}(m)\in\partial\phi_i(m)$.
By complementary slackness, $t_j^*>0$ (i.e.\ $\eta^*>C_{ij}$) implies $q_j=\beta m_j$, while $s_j^*>0$ (i.e.\ $\eta^*<C_{ij}$) implies $q_j=\alpha m_j$.
Hence $\eta^*$ coincides with the threshold cost $\tau_i(m)$ of the greedy solution in Step~3.
Setting $\tau_i(m):=\eta^*$, the constraint $s_j-t_j=C_{ij}-\tau_i(m)$ with $s_j,t_j\ge 0$ is uniquely decomposed as
\begin{align*}
	s_j^*=(C_{ij}-\tau_i(m))_+,\qquad
	t_j^*=(\tau_i(m)-C_{ij})_+,
\end{align*}
and we obtain the explicit subgradient
\begin{align*}
	g_j^{(i)}(m)
	=
	\alpha(C_{ij}-\tau_i(m))_+
	-\beta(\tau_i(m)-C_{ij})_+,
	\qquad j\in[k_v].
\end{align*}
Therefore
\begin{align*}
	|g_j^{(i)}(m)|\le \beta D_p,
	\qquad
	\|g^{(i)}(m)\|_\infty\le \beta D_p.
\end{align*}
If $i^*\in\arg\max_i \phi_i(m)$, then $g^{(i^*)}(m)\in\partial f(m)$.
Consequently, for any $m,m'\in\mathcal M$,
\begin{align*}
	|\phi_i(m)-\phi_i(m')|\le \beta D_p\|m-m'\|_1,
	\qquad
	|f(m)-f(m')|\le \beta D_p\|m-m'\|_1.
\end{align*}

We optimize $f$ directly over $\mathcal M$ using generalized-KL mirror descent.
At iteration $t$, choose $i_t\in\arg\max_i\phi_i(m^{(t)})$ and let
\begin{align*}
	g^{(t)}\in\partial\phi_{i_t}(m^{(t)})\subseteq\partial f(m^{(t)}).
\end{align*}
With stepsize $\eta_t>0$, perform
\begin{align*}
	\tilde m_i^{(t+1)} = m_i^{(t)}\exp\!\bigl(-\eta_t g_i^{(t)}\bigr),
	\qquad i\in[k_v],
\end{align*}
then project by generalized KL:
\begin{align*}
	m^{(t+1)}
	=
	\arg\min_{m\in\mathcal M} D_{\mathrm{gKL}}\!\left(m\,\middle\|\,\tilde m^{(t+1)}\right),
\end{align*}
where
\begin{align*}
	D_{\mathrm{gKL}}(u\|v)
	:=
	\sum_{i=1}^{k_v}\left(u_i\log\frac{u_i}{v_i}-u_i+v_i\right).
\end{align*}
This projection has a closed form. Let $S_t:=\sum_i\tilde m_i^{(t+1)}$. Then
for fixed total mass $s=\sum_i m_i$, KKT gives the minimizer
$m=(s/S_t)\tilde m^{(t+1)}$, and minimizing over $s\in[1/\beta,1/\alpha]$
reduces to clipping $S_t$ to this interval.
\begin{align*}
	m^{(t+1)}
	=
	\begin{cases}
		\dfrac{(1/\beta)}{S_t}\,\tilde m^{(t+1)},  & S_t<1/\beta,                \\
		\tilde m^{(t+1)},                          & 1/\beta\le S_t\le 1/\alpha, \\
		\dfrac{(1/\alpha)}{S_t}\,\tilde m^{(t+1)}, & S_t>1/\alpha.
	\end{cases}
\end{align*}
Return
\begin{align*}
	\bar m_T:=\frac{1}{T}\sum_{t=1}^T m^{(t)}.
\end{align*}

\begin{proposition}\label{prop:mirror-descent-base-measure}
	Let $m^*\in\arg\min_{m\in\mathcal M} f(m)$ and initialize
	\begin{align*}
		m^{(1)}=\frac{1}{\alpha k_v}\mathbf 1.
	\end{align*}
	Choose a constant stepsize
	\begin{align*}
		\eta_t \equiv \eta := \frac{\sqrt{2(1+\log k)}}{\beta D_p\sqrt{T}}.
	\end{align*}
	Then
	\begin{align*}
		f(\bar m_T)-f(m^*)
		\le
		\frac{\beta D_p}{\alpha}\sqrt{\frac{2(1+\log k)}{T}}
		=
		O\!\left(\frac{\beta}{\alpha}D_p\sqrt{\frac{\log k}{T}}\right).
	\end{align*}
\end{proposition}

\begin{proof}
	By construction, $g^{(t)}\in\partial f(m^{(t)})$ for every $t$.
	By the regret bound for the unnormalized exponentiated gradient algorithm
	\citep[Theorem~2.23]{shalev2025online}, for any $m\in\mathcal M$,
	\begin{align*}
		\sum_{t=1}^T \langle g^{(t)},m^{(t)}-m\rangle
		\le
		\frac{D_{\mathrm{gKL}}(m\|m^{(1)})}{\eta}
		+
		\frac{\eta}{2}\sum_{t=1}^T\sum_{i=1}^{k_v} m_i^{(t)}\bigl(g_i^{(t)}\bigr)^2.
	\end{align*}
	Using $\|g^{(t)}\|_\infty\le \beta D_p$ and $\sum_i m_i^{(t)}\le 1/\alpha$,
	\begin{align*}
		\sum_{i=1}^{k_v} m_i^{(t)}\bigl(g_i^{(t)}\bigr)^2
		\le
		\frac{(\beta D_p)^2}{\alpha}.
	\end{align*}

	Next, fix any $m\in\mathcal M$. Since $m^{(1)}=(1/(\alpha k_v))\mathbf 1$,
	\begin{align*}
		D_{\mathrm{gKL}}(m\|m^{(1)})
		 & =
		\sum_{j=1}^{k_v} m_j\log(\alpha k_v m_j)-\sum_{j=1}^{k_v} m_j+\frac{1}{\alpha}.
	\end{align*}
	Let $p_j:=\alpha m_j$, so $\sum_j p_j\le 1$. Then
	\begin{align*}
		\sum_{j=1}^{k_v} m_j\log(\alpha k_v m_j)
		 & =
		\frac{1}{\alpha}\sum_{j=1}^{k_v} p_j\log(kp_j)
		\le
		\frac{\log k}{\alpha},
	\end{align*}
	because $\sum_j p_j\log p_j\le 0$ and $\sum_j p_j\le 1$. Therefore
	\begin{align*}
		D_{\mathrm{gKL}}(m\|m^{(1)})
		\le
		\frac{1+\log k}{\alpha}.
	\end{align*}

	Applying the regret bound with $m=m^*$ and dividing by $T$ gives
	\begin{align*}
		\frac{1}{T}\sum_{t=1}^T \langle g^{(t)},m^{(t)}-m^*\rangle
		\le
		\frac{1+\log k}{\alpha\eta T}
		+
		\frac{\eta(\beta D_p)^2}{2\alpha}.
	\end{align*}
	By convexity of $f$,
	\begin{align*}
		f(m^{(t)})-f(m^*)
		\le
		\langle g^{(t)},m^{(t)}-m^*\rangle.
	\end{align*}
	Using Jensen's inequality for $\bar m_T=\frac1T\sum_{t=1}^T m^{(t)}$,
	\begin{align*}
		f(\bar m_T)-f(m^*)
		\le
		\frac{1}{T}\sum_{t=1}^T \bigl(f(m^{(t)})-f(m^*)\bigr)
		\le
		\frac{1+\log k}{\alpha\eta T}
		+
		\frac{\eta(\beta D_p)^2}{2\alpha}.
	\end{align*}
	Substituting
	\begin{align*}
		\eta=\frac{\sqrt{2(1+\log k)}}{\beta D_p\sqrt{T}}
	\end{align*}
	yields
	\begin{align*}
		f(\bar m_T)-f(m^*)
		\le
		\frac{\beta D_p}{\alpha}\sqrt{\frac{2(1+\log k)}{T}}.
	\end{align*}
\end{proof}

For computation, first sort $\{C_{ij}\}_{j=1}^{k_v}$ for each $i\in[k]$ once in $O(k^2\log k)$ time.
Then each iteration evaluates all $\phi_i$ and one active subgradient in $O(k^2)$ time, so with
\begin{align*}
	T=O\!\left(\frac{(\beta/\alpha)^2D_p^2\log k}{\varepsilon_{\mathrm{opt}}^2}\right)
\end{align*}
iterations, the total complexity is
\begin{align*}
	O\!\left(k^2\log k + k^2\frac{(\beta/\alpha)^2D_p^2\log k}{\varepsilon_{\mathrm{opt}}^2}\right).
\end{align*}

\section{Optimal Base Measure on the Sphere} \label{app:proof-sphere}
As an example of a continuous setting, we consider $\mathcal{X} = \mathcal{S}^d := \{x \in \mathbb{R}^{d+1} : \|x\|_2 = 1\}$ equipped with the Euclidean distance $d(x,y) = \|x-y\|_2$.
Utilizing the rotational symmetry of the sphere, we can derive the optimal base measure analytically.
\begin{proposition}[Optimal base measure on the sphere]\label{prop:sphere}
	Fix $p \ge 1$, and let $\sigma$ denote the uniform probability measure on $\mathcal{S}^d$.
	Then, the optimal base measure is given by
	$m^* = \alpha^* \sigma$,
	where $\alpha^*$ is a scale parameter defined as
	\begin{align*}
		\alpha^*
		=
		\frac{1}{e^{-\varepsilon/2} + \bigl(e^{\varepsilon/2}-e^{-\varepsilon/2}\bigr)S_d(t^*)},
	\end{align*}
	with
	\begin{align*}
		f_d(u)     & = C_d(1-u^2)^{\frac{d-2}{2}}, \quad
		C_d = \frac{\Gamma((d+1)/2)}{\sqrt{\pi}\Gamma(d/2)}, \\
		g_p(u)     & = (2-2u)^{p/2}, \quad
		S_d(t) = \int_t^1 f_d(u)\,\d u,                      \\
		U_{d,p}(t) & = \int_t^1 g_p(u)f_d(u)\,\d u, \quad
		H_{d,p} = \int_{-1}^1 g_p(u)f_d(u)\,\d u,
	\end{align*}
	and $t^*\in(-1,1)$ is the unique solution of
	\begin{align*}
		\bigl(e^{\varepsilon/2}-e^{-\varepsilon/2}\bigr)U_{d,p}(t^*) + e^{-\varepsilon/2}H_{d,p}
		=
		g_p(t^*)\Bigl(
		e^{-\varepsilon/2}
		+
		\bigl(e^{\varepsilon/2}-e^{-\varepsilon/2}\bigr)S_d(t^*)
		\Bigr).
	\end{align*}
	Moreover, the map $p\mapsto \alpha^*(p)$ is decreasing.
\end{proposition}
The optimal mechanism $\mathcal{M}_{m^*, \varepsilon}$ can be interpreted as
a generalization of PrivUnit~\citep{asi2022optimal}, which is an optimal LDP mechanism for privatizing unit vectors on the sphere,
except that we do not require the unbiasedness condition since our focus is not on mean estimation.
Interestingly, the scale parameter $\alpha^*$ is decreasing in $p$:
as $p$ increases, the cost of transporting mass over long distances becomes more expensive, and thus
it is preferable to reduce the minimum amount of mass ($\alpha e^{-\varepsilon/2} \sigma$) that must be assigned to locations far away.
\begin{proof}
	We divide the proof into four steps.

	\paragraph{Step 1: Reduction to rotation-invariant base measures.}\, \\
	Let $\sigma$ denote the uniform probability measure on $\mathcal{S}^d$. By Proposition~\ref{prop:dirac},
	it is enough to minimize
	\begin{align*}
		F(m):=\sup_{x\in \mathcal{S}^d} W_p\!\bigl(\delta_x,\Pi_{Q_{m,\varepsilon}}(\delta_x)\bigr).
	\end{align*}
	For $u\in \mathcal{S}^d$, define
	\begin{align*}
		\phi_u(m):=\inf_{\nu\in Q_{m,\varepsilon}}W_p^p(\delta_u,\nu).
	\end{align*}
	Since the source is a Dirac mass,
	\begin{align*}
		W_p^p(\delta_u,\nu)=\int_{\mathcal{S}^d}\|u-y\|_2^p\,d\nu(y),
	\end{align*}
	so the optimization defining $\phi_u(m)$ is linear over $Q_{m,\varepsilon}$. Fix, for each $u\in \mathcal{S}^d$,
	an optimizer $\nu_u\in Q_{m,\varepsilon}$, and define
	\begin{align*}
		\bar m(A):=\int_{SO(d+1)}(R_\#m)(A)\,d\lambda(R),
		\qquad
		\bar\nu_x(A):=\int_{SO(d+1)}(R_\#\nu_{R^{-1}x})(A)\,d\lambda(R).
	\end{align*}
	For every measurable $A\subset \mathcal{S}^d$,
	\begin{align*}
		e^{-\varepsilon/2}\bar m(A)\le \bar\nu_x(A)\le e^{\varepsilon/2}\bar m(A),
	\end{align*}
	hence $\bar\nu_x\in Q_{\bar m,\varepsilon}$. Therefore,
	\begin{align*}
		\phi_x(\bar m)\le \int_{\mathcal{S}^d}\|x-y\|_2^p\,d\bar\nu_x(y).
	\end{align*}
	Using rotational invariance of the Euclidean norm,
	\begin{align*}
		\int_{\mathcal{S}^d}\|x-y\|_2^p\,d\bar\nu_x(y)
		 & = \int_{SO(d+1)}\int_{\mathcal{S}^d}\|x-y\|_2^p\,d(R_\#\nu_{R^{-1}x})(y)\,d\lambda(R) \\
		 & = \int_{SO(d+1)}\int_{\mathcal{S}^d}\|R^{-1}x-z\|_2^p\,d\nu_{R^{-1}x}(z)\,d\lambda(R) \\
		 & = \int_{SO(d+1)}\phi_{R^{-1}x}(m)\,d\lambda(R)                                        \\
		 & \le \sup_{u\in \mathcal{S}^d}\phi_u(m).
	\end{align*}
	Hence
	\begin{align*}
		\sup_{x\in \mathcal{S}^d}\phi_x(\bar m)\le \sup_{u\in \mathcal{S}^d}\phi_u(m).
	\end{align*}
	Since $F(m)^p=\sup_{x\in \mathcal{S}^d}\phi_x(m)$, we obtain $F(\bar m)\le F(m)$.
	Thus an optimal base measure can be taken rotation-invariant, i.e.,
	\begin{align*}
		m = \alpha \sigma
	\end{align*}
	for some $\alpha>0$.

	\paragraph{Step 2: Structure of the optimal projection.}\, \\
	Fix $x\in \mathcal{S}^d$ and $\alpha>0$. Any $\nu\in Q_{\alpha\sigma,\varepsilon}$ is absolutely continuous
	with respect to $\sigma$ with density $q=d\nu/d\sigma$ satisfying
	\begin{align*}
		\alpha e^{-\varepsilon/2}\le q(y)\le \alpha e^{\varepsilon/2}
	\end{align*}
	for $\sigma$-almost every $y$, and
	\begin{align*}
		\int_{\mathcal{S}^d} q(y)\,d\sigma(y) = 1.
	\end{align*}
	Since
	\begin{align*}
		W_p^p(\delta_x,\nu) = \int_{\mathcal{S}^d}\|x-y\|_2^p q(y)\,d\sigma(y),
	\end{align*}
	minimizing the Wasserstein cost is equivalent to minimizing the above linear functional.
	Let $u=\langle x,y\rangle\in[-1,1]$ and define
	\begin{align*}
		g_p(u):=(2-2u)^{p/2}.
	\end{align*}
	Then $\|x-y\|_2^p=g_p(\langle x,y\rangle)$, and $g_p$ is strictly decreasing on $[-1,1]$.
	By the bathtub principle, the minimizer saturates the upper bound on a spherical cap and the
	lower bound on its complement. Thus, for some threshold $t\in[-1,1]$,
	\begin{align*}
		q_{x,t}(y)
		= \alpha e^{\varepsilon/2}\1_{\{\langle x,y\rangle\ge t\}}
		+ \alpha e^{-\varepsilon/2}\1_{\{\langle x,y\rangle< t\}}.
	\end{align*}

	Let
	\begin{align*}
		S_d(t) := \sigma\!\left(\{y\in \mathcal{S}^d:\langle x,y\rangle\ge t\}\right).
	\end{align*}
	The normalization condition $\int q_{x,t}\,d\sigma=1$ gives
	\begin{align*}
		\alpha e^{\varepsilon/2}S_d(t) + \alpha e^{-\varepsilon/2}(1-S_d(t)) = 1,
	\end{align*}
	hence
	\begin{align*}
		\alpha = \frac{1}{e^{-\varepsilon/2}+\bigl(e^{\varepsilon/2}-e^{-\varepsilon/2}\bigr)S_d(t)}.
	\end{align*}

	\paragraph{Step 3: Reduction to a one-dimensional objective.}\, \\
	Under the uniform measure $\sigma$, the random variable $u=\langle x,y\rangle$ has density
	\begin{align*}
		f_d(u) = C_d(1-u^2)^{\frac{d-2}{2}},
		\qquad
		C_d = \frac{\Gamma((d+1)/2)}{\sqrt{\pi}\Gamma(d/2)}.
	\end{align*}
	Hence
	\begin{align*}
		S_d(t)=\int_t^1 f_d(u)\,\d u.
	\end{align*}
	Define
	\begin{align*}
		U_{d,p}(t) & := \int_t^1 g_p(u)f_d(u)\,\d u,    \\
		H_{d,p}    & := \int_{-1}^1 g_p(u)f_d(u)\,\d u.
	\end{align*}
	For $\nu_{x,t}$ induced by $q_{x,t}$,
	\begin{align*}
		W_p^p(\delta_x,\nu_{x,t})
		 & = \alpha e^{\varepsilon/2}U_{d,p}(t)
		+ \alpha e^{-\varepsilon/2}\bigl(H_{d,p}-U_{d,p}(t)\bigr)                                                    \\
		 & = \alpha\Bigl(\bigl(e^{\varepsilon/2}-e^{-\varepsilon/2}\bigr)U_{d,p}(t)+e^{-\varepsilon/2}H_{d,p}\Bigr).
	\end{align*}
	Let
	\begin{align*}
		A:=e^{\varepsilon/2}-e^{-\varepsilon/2},
		\qquad
		B:=e^{-\varepsilon/2}.
	\end{align*}
	Substituting $\alpha=1/(B+A S_d(t))$, we get the one-dimensional objective
	\begin{align*}
		J(t):=W_p^p(\delta_x,\nu_{x,t})=\frac{A U_{d,p}(t)+B H_{d,p}}{B+A S_d(t)}.
	\end{align*}

	\paragraph{Step 4: Optimality condition and uniqueness.}\, \\
	Write $U(t):=U_{d,p}(t)$, $S(t):=S_d(t)$, $H:=H_{d,p}$, and $g(t):=g_p(t)$.
	Since
	\begin{align*}
		U'(t)=-g(t)f_d(t),
		\qquad
		S'(t)=-f_d(t),
	\end{align*}
	differentiation yields
	\begin{align*}
		J'(t)
		= \frac{A f_d(t)}{(B+A S(t))^2}\,G(t),
	\end{align*}
	where
	\begin{align*}
		G(t):=A U(t)+B H-g(t)\bigl(B+A S(t)\bigr).
	\end{align*}
	Also,
	\begin{align*}
		G'(t)
		 & = A U'(t)-g'(t)\bigl(B+A S(t)\bigr)-g(t)A S'(t) \\
		 & = -g'(t)\bigl(B+A S(t)\bigr).
	\end{align*}
	Because $g$ is strictly decreasing, $g'(t)<0$ on $(-1,1)$; and $B+A S(t)>0$.
	Therefore $G'(t)>0$ on $(-1,1)$, so $G$ is strictly increasing.

	At $t=-1$ and $t=1$,
	\begin{align*}
		G(-1)
		 & = (A+B)H-g(-1)(A+B)
		= (A+B)\bigl(H-2^p\bigr) < 0, \\
		G(1)
		 & = B H > 0,
	\end{align*}
	since $g(-1)=2^p$ and $H<2^p$.
	Thus, by continuity and strict monotonicity, there exists a unique $t^*\in(-1,1)$ such that
	\begin{align*}
		G(t^*)=0,
	\end{align*}
	i.e.
	\begin{align*}
		\bigl(e^{\varepsilon/2}-e^{-\varepsilon/2}\bigr)U_{d,p}(t^*) + e^{-\varepsilon/2}H_{d,p}
		= g_p(t^*)\Bigl(e^{-\varepsilon/2}+\bigl(e^{\varepsilon/2}-e^{-\varepsilon/2}\bigr)S_d(t^*)\Bigr).
	\end{align*}
	Since $A>0$, $f_d(t)>0$ on $(-1,1)$, and $(B+A S(t))^2>0$, we have
	\begin{align*}
		\operatorname{sign}(J'(t)) = \operatorname{sign}(G(t)).
	\end{align*}
	Hence $J$ is strictly decreasing on $(-1,t^*)$ and strictly increasing on $(t^*,1)$.
	Therefore $t^*$ is the unique global minimizer of $J$, equivalently of
	$W_p^p(\delta_x,\nu_{x,t})$.

	Finally, substituting $t^*$ into the normalization formula gives
	\begin{align*}
		\alpha^*
		= \frac{1}{e^{-\varepsilon/2}+\bigl(e^{\varepsilon/2}-e^{-\varepsilon/2}\bigr)S_d(t^*)}.
	\end{align*}
	Thus $m^*=\alpha^*\sigma$ is optimal, with $t^*$ uniquely characterized by the
	stated scalar equation.

	\paragraph{Step 5: Monotonicity of $p\mapsto\alpha^*(p)$.}\, \\
	From the formula for $\alpha^*$ and the monotonicity of $S_d(t)$,
	it suffices to show that $p\mapsto t^*(p)$ is decreasing on $[1,\infty)$.
	Write
	\begin{align*}
		c_\pm := e^{\pm \varepsilon/2},
		\qquad
		r(u):=\sqrt{2-2u},
	\end{align*}
	so that $g_p(u)=r(u)^p$.

	For $t\in(-1,1)$, define
	\begin{align*}
		Z(t):=c_-+(c_+-c_-)S_d(t),
	\end{align*}
	and the probability measure
	\begin{align*}
		\mu_t(\d u)
		:=
		\frac{\bigl(c_+\1_{\{u\ge t\}}+c_-\1_{\{u<t\}}\bigr)f_d(u)\,\d u}{Z(t)}.
	\end{align*}
	For each $p\ge 1$, let $t^*(p)$ be the unique solution from Step~4.
	Its defining equation is equivalent to
	\begin{align*}
		\int_{-1}^1 g_p(u)\,\mu_{t^*(p)}(\d u)=g_p(t^*(p)),
	\end{align*}
	equivalently,
	\begin{align*}
		\left(\int_{-1}^1 r(u)^p\,\mu_{t^*(p)}(\d u)\right)^{1/p}=r\bigl(t^*(p)\bigr).
	\end{align*}

	Define
	\begin{align*}
		M_p(t):=\left(\int_{-1}^1 r(u)^p\,\mu_t(\d u)\right)^{1/p},
		\qquad
		F_p(t):=M_p(t)-r(t).
	\end{align*}
	Then $t^*(p)$ is the unique zero of $F_p$.

	Fix $t\in(-1,1)$. Since $c_\pm>0$, $f_d(u)>0$ for $u\in(-1,1)$, and $Z(t)>0$,
	$\mu_t$ has positive density on $(-1,1)$. As $r$ is positive and nonconstant on $(-1,1)$,
	strict monotonicity of power means implies
	\begin{align*}
		p_1<p_2
		\quad\Longrightarrow\quad
		M_{p_1}(t)<M_{p_2}(t).
	\end{align*}

	Let $t_i:=t^*(p_i)$ for $1\le p_1<p_2$. Since $M_{p_1}(t_1)=r(t_1)$,
	\begin{align*}
		M_{p_2}(t_1)>M_{p_1}(t_1)=r(t_1),
	\end{align*}
	hence $F_{p_2}(t_1)>0$.

	Also, $F_p$ is continuous in $t$, and by dominated convergence,
	\begin{align*}
		\lim_{t\downarrow-1}F_p(t)=H_{d,p}^{1/p}-2<0,
		\qquad
		\lim_{t\uparrow1}F_p(t)=H_{d,p}^{1/p}>0.
	\end{align*}
	(Indeed, $r(u)\le 2$ and $r(u)<2$ for $u>-1$, so $H_{d,p}^{1/p}<2$.)
	Since $F_{p_2}$ has a unique zero, $F_{p_2}(t_1)>0$ implies
	\begin{align*}
		t^*(p_2)<t^*(p_1).
	\end{align*}
	Thus $p\mapsto t^*(p)$ is strictly decreasing on $[1,\infty)$.
	This completes the proof.
\end{proof}

\section{LP Reformulation as Minimum-Cost Flow}\label{app:lp-projection}
We show that the LP formulation of the Wasserstein projection in Section~\ref{sec:practical-implementation}
can be reformulated as a minimum-cost flow problem on a bipartite network.

Recall the LP:
\begin{align*}
	\min_{\nu \in \R_+^{k_v}, \, \pi \in \R_+^{k\times k_v}} & \sum_{i=1}^{k}\sum_{j=1}^{k_v} C_{ij} \pi_{ij}                                                                                  \\
	\text{subject to}                                        & \sum_{j=1}^{k_v} \pi_{ij} = \mu_i \text{ for any } i\in [k], \quad \sum_{i=1}^{k} \pi_{ij} = \nu_j \text{ for any } j\in [k_v], \\
	                                                         & e^{-\varepsilon/2} m_j \leq \nu_j \leq e^{\varepsilon/2} m_j \text{ for any } j\in [k_v].
\end{align*}

We construct a directed graph $G = (V, E)$ as follows.
\paragraph{Nodes.}
The node set is $V = \{s\} \cup \{u_i : i \in [k]\} \cup \{v_j : j \in [k_v]\} \cup \{t\}$,
where $s$ is the source and $t$ is the sink.

\paragraph{Edges and capacities.}
\begin{itemize}
	\item \textbf{Supply edges:} For each $i \in [k]$, add edge $(s, u_i)$ with capacity $[\mu_i, \mu_i]$ (i.e., the flow is fixed to $\mu_i$) and cost $0$.
	\item \textbf{Transport edges:} For each $i \in [k]$ and $j \in [k_v]$, add edge $(u_i, v_j)$ with capacity $[0, \mu_i]$ and cost $C_{ij}$.
	\item \textbf{Demand edges:} For each $j \in [k_v]$, add edge $(v_j, t)$ with capacity $[e^{-\varepsilon/2} m_j,\, e^{\varepsilon/2} m_j]$ and cost $0$.
\end{itemize}
The total supply at $s$ equals $\sum_i \mu_i = 1$, and the demand edges enforce that the total flow into $t$ also equals $1$
(since $Q_{m,\varepsilon} \neq \emptyset$ implies $e^{-\varepsilon/2} \sum_j m_j \le 1 \le e^{\varepsilon/2} \sum_j m_j$).

\paragraph{Correspondence.}
Let $f$ be a feasible flow on $G$.
Setting $\pi_{ij} = f(u_i, v_j)$ and $\nu_j = f(v_j, t)$ recovers a feasible solution of the LP, and vice versa.
The objective $\sum_{i,j} C_{ij} f(u_i, v_j)$ coincides with the LP objective.
Thus the LP is equivalent to a minimum-cost flow problem on $G$.

\paragraph{Complexity.}
The network has $|V| = k + k_v + 2$ nodes and $|E| = k + k \cdot k_v + k_v$ edges.
This is analogous to the standard network-flow formulation of the discrete optimal transport problem,
where the only difference is that the demand edges carry box constraints from the LDP polytope
instead of fixed demands.
Via the successive shortest path algorithm, the minimum-cost flow problem can be solved in $\tilde O(k^3)$ time~\citep{ahuja2014network}.

\section{Proof of Lemma~\ref{lem:kl-projection}} \label{app:proof-kl-projection}
\begin{proof}
	Let $C \in \mathbb{R}_+^{k\times k_v}$ be the cost matrix.
	Recall the entropically regularized transport cost:
	\begin{align*}
		\mathsf{OT}_\lambda(\mu,\nu)
		=
		\min_{\substack{\pi \in \Delta_{k\times k_v} \\
				\pi \1=\mu,\ \pi^\top \1=\nu}}
		\sum_{i,j} C_{ij}\pi_{ij} + \lambda \sum_{i,j} \pi_{ij} \log \pi_{ij}.
	\end{align*}
	Hence
	\begin{align*}
		\mathcal{M}^{\lambda}_{m,\varepsilon}[\mu]
		 & = \arg\min_{\nu \in Q_{m,\varepsilon}} \mathsf{OT}_\lambda(\mu,\nu) \\
		 & = \arg\min_{\substack{\pi \in \Delta_{k\times k_v}                  \\
				\pi \1=\mu,\ \pi^\top \1\in Q_{m,\varepsilon}}}
		\sum_{i,j} C_{ij}\pi_{ij} + \lambda \sum_{i,j} \pi_{ij} \log \pi_{ij}.
	\end{align*}

	Since
	\begin{align*}
		\log K_{ij} = -\frac{C_{ij}}{\lambda} - \log Z_\lambda,
	\end{align*}
	we have
	\begin{align*}
		\mathrm{KL}(\pi\|K)
		 & = \sum_{i,j} \pi_{ij}\log\left(\frac{\pi_{ij}}{K_{ij}}\right) \\
		 & = \sum_{i,j}\pi_{ij}\log\pi_{ij}
		+ \frac{1}{\lambda}\sum_{i,j} C_{ij}\pi_{ij}
		+ (\log Z_\lambda)\sum_{i,j}\pi_{ij}
		.
	\end{align*}
	On the constraint set $\pi \1=\mu$, the total mass is fixed and equals
	\begin{align*}
		\sum_{i,j}\pi_{ij} = \sum_i \mu_i = 1.
	\end{align*}
	Therefore minimizing
	\begin{align*}
		\sum_{i,j} C_{ij}\pi_{ij} + \lambda \sum_{i,j} \pi_{ij} \log \pi_{ij}
	\end{align*}
	is equivalent (up to constants independent of $\pi$) to minimizing $\mathrm{KL}(\pi\|K)$.
	This proves
	\begin{align*}
		\pi_* = \arg\min_{\pi \in R_\mu \cap R_{m,\varepsilon}} \mathrm{KL}(\pi\|K).
	\end{align*}
	By construction, $\nu_* := {\pi_*}^\top \1$ belongs to $Q_{m,\varepsilon}$ and attains
	\begin{align*}
		\min_{\nu \in Q_{m,\varepsilon}} \mathsf{OT}_\lambda(\mu,\nu),
	\end{align*}
	hence $\mathcal{M}^{\lambda}_{m,\varepsilon}[\mu] = \nu_*$.

	Next we derive the closed-form KL projections.

	\paragraph{Projection onto $R_\mu$:}
	\begin{align*}
		P_{R_\mu}(\pi)
		= \arg\min_{\hat\pi \ge 0} \mathrm{KL}(\hat\pi\|\pi)
		\quad \text{s.t.}\quad
		\sum_j \hat\pi_{ij}=\mu_i\ \ (\forall i).
	\end{align*}
	The Lagrangian optimality condition gives
	\begin{align*}
		\log\left(\frac{\hat\pi_{ij}}{\pi_{ij}}\right) + a_i = 0,
	\end{align*}
	so $\hat\pi_{ij}=\pi_{ij}e^{-a_i}$. Enforcing row sums,
	\begin{align*}
		e^{-a_i} = \frac{\mu_i}{\sum_j \pi_{ij}},
	\end{align*}
	and therefore
	\begin{align*}
		P_{R_\mu}(\pi)_{ij} = \frac{\mu_i}{\sum_j \pi_{ij}}\pi_{ij}.
	\end{align*}

	\paragraph{Projection onto $R_{m,\varepsilon}$:}
	write $s_j := \sum_i \pi_{ij}$. The constraint $\pi^\top \1\in Q_{m,\varepsilon}$ is
	\begin{align*}
		e^{-\varepsilon/2}m_j \le q_j \le e^{\varepsilon/2}m_j,
		\qquad \sum_j q_j = 1,
		\quad (\forall j),
	\end{align*}
	where $q_j := \sum_i \hat\pi_{ij}$.
	For fixed $q$, we have the decomposition
	\begin{align*}
		\mathrm{KL}(\hat\pi\|\pi)
		 & = \sum_{i,j} \hat\pi_{ij} \log\left(\frac{\hat\pi_{ij}}{\pi_{ij}}\right)        \\
		 & = \sum_{i,j} \hat\pi_{ij} \log\left(\frac{\hat\pi_{ij}}{q_j\pi_{ij}/s_j}\right)
		+ \sum_j q_j\log\left(\frac{q_j}{s_j}\right)
		\ge \sum_j q_j\log\left(\frac{q_j}{s_j}\right),
	\end{align*}
	where equality holds iff $\hat\pi_{ij}=q_j\pi_{ij}/s_j$.
	Therefore, the column-marginal projection reduces to
	\begin{align*}
		q^* = \arg\min_{q \in Q_{m,\varepsilon}} \sum_j q_j\log\left(\frac{q_j}{s_j}\right).
	\end{align*}
	Let $\alpha := e^{-\varepsilon/2}$ and $\beta := e^{\varepsilon/2}$.
	The KKT conditions for
	\begin{align*}
		\min_q\quad & \sum_j q_j\log\left(\frac{q_j}{s_j}\right) \\
		\text{s.t.}\quad
		            & \sum_j q_j = 1,\qquad
		\alpha m_j \le q_j \le \beta m_j\ (\forall j)
	\end{align*}
	yield
	\begin{align*}
		q_j^* = \min\{\max\{e^{\theta}s_j,\alpha m_j\},\beta m_j\},
	\end{align*}
	where $\theta \in \mathbb{R}$ is chosen so that $\sum_j q_j^*=1$.
	Define
	\begin{align*}
		\psi(\theta) := \sum_j \min\{\max\{e^{\theta}s_j,\alpha m_j\},\beta m_j\}.
	\end{align*}
	Then $\psi$ is continuous and nondecreasing, with
	\begin{align*}
		\lim_{\theta\to-\infty}\psi(\theta)=\alpha\sum_j m_j,
		\qquad
		\lim_{\theta\to+\infty}\psi(\theta)=\beta\sum_j m_j.
	\end{align*}
	Since $Q_{m,\varepsilon}$ is nonempty, $\alpha\sum_j m_j \le 1 \le \beta\sum_j m_j$, hence there exists at least one $\theta$ such that $\psi(\theta)=1$.
	Finally,
	\begin{align*}
		P_{R_{m,\varepsilon}}(\pi)_{ij}
		= \frac{q_j^*}{\sum_i \pi_{ij}}\pi_{ij}.
	\end{align*}
	This completes the proof.
\end{proof}

\section{Proof of Proposition~\ref{prop:entropic-gap}} \label{app:proof-entropic-gap}
\begin{proof}
	Let $Q:=Q_{m,\varepsilon}$ and $N:=k k_v$.

	For any $\nu\in\mathcal P([k_v])$ and any coupling $\pi\in\Pi(\nu, \mu)$ with
	$\sum_{i,j}\pi_{ij}=1$, we have
	\begin{align*}
		-\log N \le \sum_{i,j}\pi_{ij}\log\pi_{ij} \le 0,
	\end{align*}
	where $0\log 0:=0$.
	Therefore,
	\begin{align*}
		\sum_{i,j} C_{ij}\pi_{ij}-\lambda\log N
		\le
		\sum_{i,j} C_{ij}\pi_{ij}+\lambda\sum_{i,j}\pi_{ij}\log\pi_{ij}
		\le
		\sum_{i,j} C_{ij}\pi_{ij}.
	\end{align*}
	Taking the minimum over $\pi\in\Pi(\nu, \mu)$ yields, for every $\nu$,
	\begin{align}
		\mathsf{OT}(\nu,\mu)-\lambda\log N
		\le
		\mathsf{OT}_\lambda(\nu,\mu)
		\le
		\mathsf{OT}(\nu,\mu).
		\label{eq:ot-entropic-sandwich}
	\end{align}

	Let
	\begin{align*}
		\nu^* & \in \arg\min_{\nu\in Q} W_p(\nu,\mu),
		\qquad
		\nu^\lambda := \mathcal M^{\lambda}_{m,\varepsilon}[\mu].
	\end{align*}
	Since $x\mapsto x^{1/p}$ is strictly increasing on $\mathbb R_+$, we also have
	\begin{align*}
		\nu^*\in\arg\min_{\nu\in Q}\mathsf{OT}(\nu,\mu).
	\end{align*}
	By definition of $\mathcal M^{\lambda}_{m,\varepsilon}[\mu]$,
	\begin{align*}
		\nu^\lambda\in\arg\min_{\nu\in Q}\mathsf{OT}_\lambda(\nu,\mu).
	\end{align*}
	Using~\eqref{eq:ot-entropic-sandwich},
	\begin{align*}
		\mathsf{OT}(\nu^\lambda,\mu)
		 & \le \mathsf{OT}_\lambda(\nu^\lambda,\mu)+\lambda\log N \\
		 & \le \mathsf{OT}_\lambda(\nu^*,\mu)+\lambda\log N       \\
		 & \le \mathsf{OT}(\nu^*,\mu)+\lambda\log N.
	\end{align*}
	On the other hand, by optimality of $\nu^*$ for $\mathsf{OT}$ over $Q$,
	\begin{align*}
		\mathsf{OT}(\nu^*,\mu) \le \mathsf{OT}(\nu^\lambda,\mu).
	\end{align*}
	Hence
	\begin{align*}
		0
		\le
		\mathsf{OT}(\nu^\lambda,\mu)-\mathsf{OT}(\nu^*,\mu)
		\le
		\lambda\log N.
	\end{align*}

	Finally, let
	\begin{align*}
		a:=W_p^p(\nu^\lambda,\mu),\qquad b:=W_p^p(\nu^*,\mu),
	\end{align*}
	so $0\le a-b\le\lambda\log N$.
	Because $a=b+(a-b)$ and $x\mapsto x^{1/p}$ is subadditive on $\mathbb R_+$ for $p\ge1$,
	\begin{align*}
		W_p(\nu^\lambda,\mu)-W_p(\nu^*,\mu)
		 & = a^{1/p}-b^{1/p}          \\
		 & \le (a-b)^{1/p}            \\
		 & \le (\lambda\log N)^{1/p}.
	\end{align*}
	This proves the claim.
\end{proof}

\section{Proof of Theorem~\ref{thm:convergence}} \label{app:proof-convergence}
First, we show some technical lemmas.
\begin{lemma}[Nonexpansiveness of simplex-projection ratio]\label{lem:ratio-nonexpansive}
	Assume $m_j>0$ for all $j\in[k_v]$, and define
	\begin{align*}
		\Pi(x) := \arg\min_{q \in Q_{m,\varepsilon}} \mathrm{KL}(q\|x),
		\qquad
		F(x) := \frac{\Pi(x)}{x}
	\end{align*}
	for $x \in \mathbb{R}^{k_v}_{++}$. Then for all $x,y \in \mathbb{R}^{k_v}_{++}$,
	\begin{align*}
		d_H(F(x),F(y)) \le d_H(x,y).
	\end{align*}
\end{lemma}

\begin{proof}
	Let $\ell_i := e^{-\varepsilon/2}m_i$ and $u_i := e^{\varepsilon/2}m_i$.
	From the definition of $\Pi$, we have
	\begin{align*}
		\log F(x)_i & = \log \Pi(x)_i - \log x_i                                                                                         \\
		            & = \clip(\log x_i + \theta(x), [\log \ell_i, \log u_i]) - \log x_i                                                  \\
		            & = \clip(\underbrace{\log x_i + \theta(x)}_{=: a(x)_i}, [\log \ell_i, \log u_i]) - \log x_i - \theta(x) + \theta(x) \\
		            & = \phi_i(a(x)_i) - a(x)_i + \theta(x)                                                                              \\
		            & = \psi_i(a(x)_i) + \theta(x),
	\end{align*}
	where $\phi_i(z) := \clip(z, [\log \ell_i, \log u_i])$ is monotonically nondecreasing and $1$-Lipschitz,
	and $\psi_i(z) := \phi_i(z) - z$ is monotonically nonincreasing and $1$-Lipschitz.

	Let $m = \min_i\{a(x)_i - a(y)_i\}$ and $M = \max_i\{a(x)_i - a(y)_i\}$.
	From the invariance of Hilbert's metric under scaling, we have
	\begin{align*}
		d_H(x, y)       & = d_H(\exp(\log x) \cdot \exp(\theta(x)), \exp(\log y) \cdot \exp(\theta(y))) \\
		                & = d_H(\exp(a(x)), \exp(a(y)))                                                 \\
		                & = M - m,                                                                      \\
		d_H(F(x), F(y)) & = d_H(\exp(\psi(a(x)) + \theta(x)), \exp(\psi(a(y)) + \theta(y)))             \\
		                & = d_H(\exp(\psi(a(x))), \exp(\psi(a(y)))).
	\end{align*}
	Thus, it suffices to show that $\psi_i(a(x)_i) - \psi_i(a(y)_i) \in [-M, -m]$.

	\paragraph{Case 1: $m \leq 0 \leq M$.}
	If $a(x)_i \geq a(y)_i$, we have $\psi_i(a(x)_i) - \psi_i(a(y)_i) \leq 0 \leq -m$ since $\psi_i$ is nonincreasing.
	In addition, since $\psi_i$ is $1$-Lipschitz, $\psi_i(a(x)_i) - \psi_i(a(y)_i) \geq -\abs{a(x)_i - a(y)_i} = -(a(x)_i - a(y)_i) \geq -M$.
	If $a(x)_i < a(y)_i$, we have $\psi_i(a(x)_i) - \psi_i(a(y)_i) \geq 0$ since $\psi_i$ is nonincreasing, and $\psi_i(a(x)_i) - \psi_i(a(y)_i) \leq \abs{a(x)_i - a(y)_i} = a(y)_i - a(x)_i \leq -m$.

	\paragraph{Case 2: $0 < m \leq M$.}
	In this case, $a(x)_i$ is always greater than $a(y)_i$.
	From the monotonicity of $\phi_i$, we have $\phi(a(x))_i \geq \phi(a(y))_i$,
	which implies $\sum \Pi(x)_i \geq \sum \Pi(y)_i$ since $\Pi(x)_i = \exp(\phi(a(x))_i)$ and $\Pi(y)_i = \exp(\phi(a(y))_i)$.
	On the other hand, from the definition of $\Pi$, we have $\sum \Pi(x)_i = \sum \Pi(y)_i = 1$.
	Therefore, $\phi(a(x))_i = \phi(a(y))_i$ for every $i$, and hence $\psi(a(x))_i - \psi(a(y))_i = a(y)_i - a(x)_i \in [-M, -m]$.

	\paragraph{Case 3: $m \leq M < 0$.}
	Similar to Case 2, we have $\phi(a(x))_i \leq \phi(a(y))_i$ for every $i$, and hence $\psi(a(x))_i - \psi(a(y))_i = a(y)_i - a(x)_i \in [-M, -m]$.

	Combining the three cases completes the proof.
\end{proof}

\begin{lemma}[Birkhoff contraction~\citep{birkhoff1957extensions}]
	Let $A \in \mathbb{R}^{n\times m}_{++}$. Then
	\begin{align*}
		d_H(Ax,Ay) \leq \tau(A)\, d_H(x,y),
	\end{align*}
	where
	\begin{align*}
		\tau(A) = \tanh\left(\frac{\Delta(A)}{4}\right) < 1, \qquad
		\Delta(A) = \log \max_{i,j,k,l} \frac{A_{ik} A_{jl}}{A_{il} A_{jk}}.
	\end{align*}
\end{lemma}

\begin{lemma}[Contraction of the composite map]\label{lem:map-contraction}
	Define
	\begin{align*}
		T(v) := \frac{\Pi\bigl(K^\top(\mu/(Kv))\bigr)}{K^\top(\mu/(Kv))}.
	\end{align*}
	Then for all $v,w\in\mathbb{R}_{++}^{k_v}$,
	\begin{align*}
		d_H(T(v),T(w)) \le \tau(K^\top)\tau(K)\,d_H(v,w).
	\end{align*}
\end{lemma}

\begin{proof}
	Let $v,w \in \mathbb{R}_{++}^{k_v}$ and define $u(v):=\mu/(Kv)$.
	Since $\log u(v)=\log\mu-\log(Kv)$, we have
	\begin{align*}
		d_H(u(v),u(w)) = d_H(Kv,Kw).
	\end{align*}
	By the Birkhoff contraction lemma,
	\begin{align*}
		d_H(Kv,Kw) \le \tau(K) d_H(v,w),
		\qquad
		d_H(K^\top u(v),K^\top u(w)) \le \tau(K^\top) d_H(u(v),u(w)).
	\end{align*}
	By Lemma~\ref{lem:ratio-nonexpansive},
	\begin{align*}
		d_H\!\left(\frac{\Pi(K^\top u(v))}{K^\top u(v)},\frac{\Pi(K^\top u(w))}{K^\top u(w)}\right)
		\le d_H(K^\top u(v),K^\top u(w)).
	\end{align*}
	Combining the three inequalities yields the claim.
\end{proof}

\begin{lemma}[Hilbert-Lipschitz continuity of $\Pi$]\label{lem:pi-lipschitz}
	Under the assumptions of Lemma~\ref{lem:ratio-nonexpansive}, for all $x,y \in \mathbb{R}_{++}^{k_v}$,
	\begin{align*}
		d_H(\Pi(x),\Pi(y)) \le 2d_H(x,y).
	\end{align*}
\end{lemma}

\begin{proof}
	Let $F(x) := \Pi(x)/x$ as in Lemma~\ref{lem:ratio-nonexpansive}, so $\Pi(x)=x\odot F(x)$ componentwise.
	For $a,b,c,d\in\mathbb{R}_{++}^{k_v}$,
	\begin{align*}
		\max_i \frac{a_i b_i}{c_i d_i}
		 & \le \left(\max_i \frac{a_i}{c_i}\right)\left(\max_i \frac{b_i}{d_i}\right),
		\\
		\min_i \frac{a_i b_i}{c_i d_i}
		 & \ge \left(\min_i \frac{a_i}{c_i}\right)\left(\min_i \frac{b_i}{d_i}\right).
	\end{align*}
	Hence
	\begin{align*}
		d_H(a\odot b,c\odot d) \le d_H(a,c)+d_H(b,d).
	\end{align*}
	Applying this with $(a,b,c,d)=(x,F(x),y,F(y))$ and using Lemma~\ref{lem:ratio-nonexpansive},
	\begin{align*}
		d_H(\Pi(x),\Pi(y))
		 & = d_H(x\odot F(x),y\odot F(y)) \\
		 & \le d_H(x,y)+d_H(F(x),F(y))    \\
		 & \le 2d_H(x,y).
	\end{align*}
\end{proof}

\begin{proof}[Proof of Theorem~\ref{thm:convergence}]
	By Lemma~\ref{lem:map-contraction}, $T$ is a contraction in Hilbert projective metric with constant
	\begin{align*}
		c := \tau(K^\top)\tau(K) < 1.
	\end{align*}
	Hence $T$ has a unique fixed point in projective space (equivalently, unique up to positive scaling), and
	for any initial $v^{(0)}\in\mathbb{R}^{k_v}_{++}$,
	\begin{align*}
		d_H(v^{(t)},v^*) \le c^t d_H(v^{(0)},v^*).
	\end{align*}
	This proves linear convergence of $\{v^{(t)}\}$.

	It remains to show the KL bound for the output distributions.
	Let
	\begin{align*}
		s(v) := K^\top\!\left(\frac{\mu}{Kv}\right).
	\end{align*}
	By definition of $T$, for any fixed point $v^*$ we have
	\begin{align*}
		v^* = T(v^*) = \frac{\Pi(s^*)}{s^*},
		\qquad s^* := s(v^*),
	\end{align*}
	and with $q^* := \Pi(s^*)$ this gives $v^*=q^*/s^*$ componentwise.

	Using $d_H(\mu/(Kv),\mu/(Kw))=d_H(Kv,Kw)$ and Birkhoff contraction twice,
	\begin{align*}
		d_H(s^{(t)},s^*)
		 & = d_H\!\left(K^\top\!\left(\frac{\mu}{Kv^{(t)}}\right),K^\top\!\left(\frac{\mu}{Kv^*}\right)\right) \\
		 & \le c\, d_H(v^{(t)},v^*)
		\le c^{t+1} d_H(v^{(0)},v^*).
	\end{align*}
	Since $q^{(t)}=\Pi(s^{(t)})$, Lemma~\ref{lem:pi-lipschitz} yields
	\begin{align*}
		d_H(q^{(t)},q^*)
		 & = d_H(\Pi(s^{(t)}),\Pi(s^*)) \\
		 & \le 2d_H(s^{(t)},s^*)
		\le 2c^{t+1} d_H(v^{(0)},v^*).
	\end{align*}

	Finally, because $m_j>0$ and $q^{(t)},q^* \in Q_{m,\varepsilon}$, all components of $q^{(t)},q^*$ are strictly positive.
	Let $r_i := q_i^{(t)}/q_i^*$. Since both vectors are probability distributions,
	\begin{align*}
		\min_i r_i \le 1 \le \max_i r_i.
	\end{align*}
	Therefore,
	\begin{align*}
		\mathrm{KL}(q^{(t)}\|q^*)
		 & = \sum_i q_i^{(t)}\log r_i
		\le \log\!\left(\max_i r_i\right)
		\le \log\!\left(\frac{\max_i r_i}{\min_i r_i}\right)
		= d_H(q^{(t)},q^*),
	\end{align*}
	and the KL bound follows by combining with the previous inequality.
\end{proof}

\section{Additional Experiments} \label{app:additional-experiments}
\subsection{Experiments on Geometric Data}
Here, we present additional experimental results on Gowalla location dataset~\citep{cho2011friendship}.
In the Gowalla dataset, each user has a set of check-in locations, which we regard as a private distribution.
We consider a scenario where each user privately releases a check-in location and a central server aggregates the released locations to estimate the overall distribution of check-ins.
We use check-in data of 500 users in Austin, Texas and
partitioned the area into a $20 \times 20$ grid and adopted the Euclidean distance as a metric.
Then, we evaluate the performance by the Wasserstein distance ($p=1$) between the true distribution of check-ins and the estimated distribution from the released locations.
For WPM, we used, as the base measure, the one that minimizes the worst-case cost among uniform measures.
We report the mean and standard deviation over 5 runs with different choice of check-ins.

Fig.~\ref{fig:gowalla} illustrates the comparison of the proposed WPM with KPM.
We see that WPM consistently outperforms KPM across different privacy budgets.
In addition, we also provide examples of the estimated distribution of check-ins for $\varepsilon=4$ in Fig.~\ref{fig:gowalla-distribution}.
We observe that WPM captures the true distribution more accurately than KPM, which demonstrates the advantage of WPM in preserving the geometric structure of the data.

\begin{figure}[t]
	\centering
	\begin{minipage}{0.3\textwidth}
		\centering
		\includegraphics[width=0.99\textwidth]{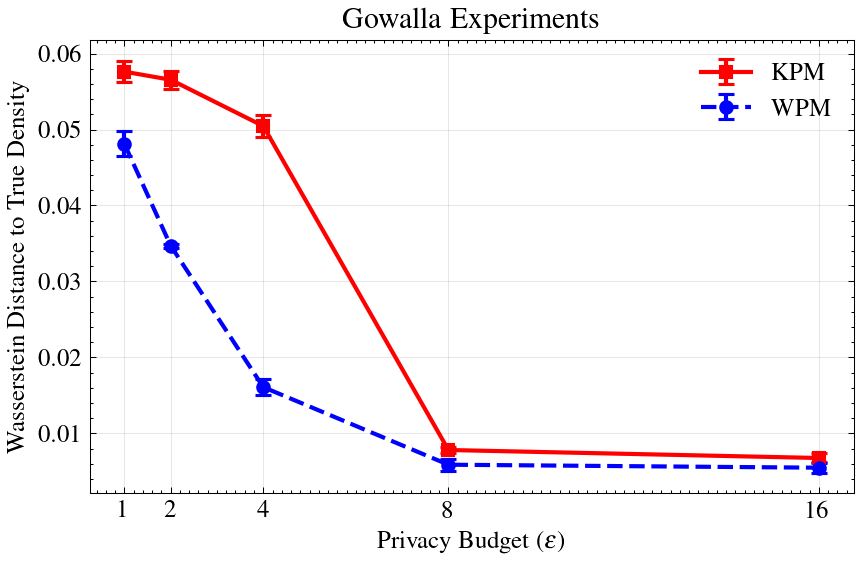}
		\caption{Comparison of WPM and KPM on Gowalla dataset.}
		\label{fig:gowalla}
	\end{minipage}
	\hfill
	\begin{minipage}{0.68\textwidth}
		\centering
		\includegraphics[width=0.99\textwidth]{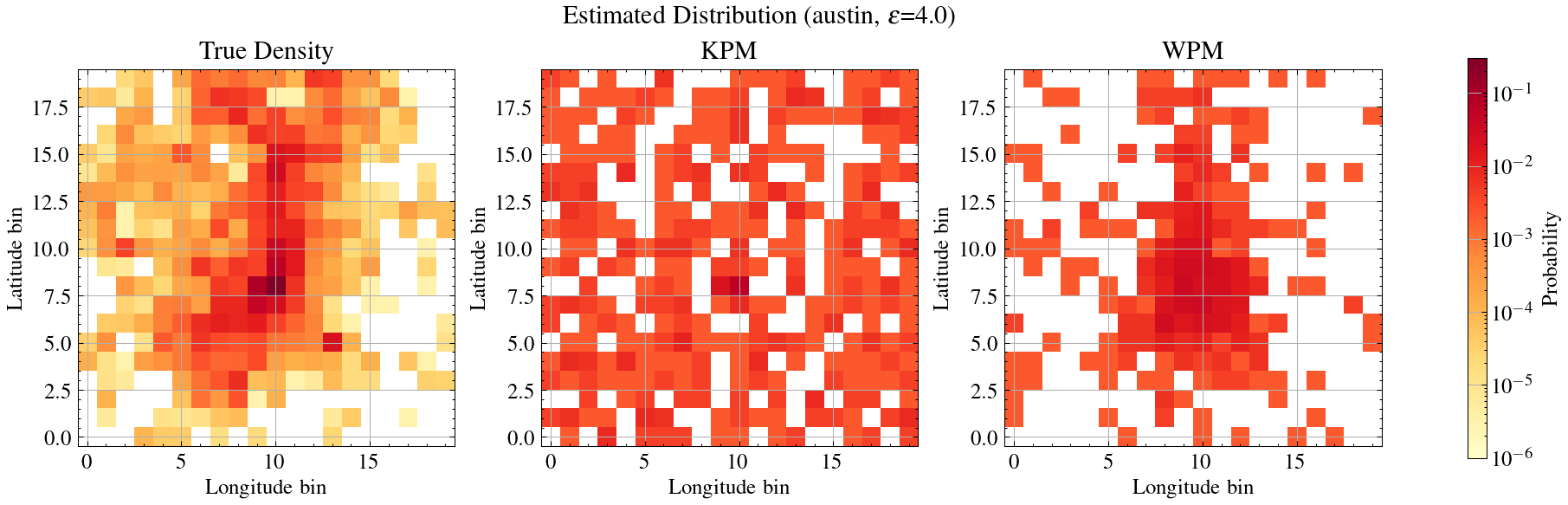}
		\caption{Estimated distribution of check-ins for $\varepsilon=4$. The true distribution is shown in the leftmost panel, and the distributions estimated by WPM and KPM are shown in the middle and right panels, respectively.}
		\label{fig:gowalla-distribution}
	\end{minipage}
\end{figure}

\section{Experimental Details} \label{app:experimental-details}
\paragraph{Datasets}
Movielens-100K~\citep{harper2015movielens} contains 100,000 ratings from 1000 users on 1700 movies,
which is distributed under a usage license that permits research use.
We randomly divided the movies into 80\% training and 20\% test sets:
the training movies are used to construct each user's input distribution
(proportional to ratings), while the test movies form the candidate set
for recommendation.
Users are split into 80\% validation and 20\% test sets;
the validation users are used for hyperparameter tuning,
and the test users for final evaluation.
Gowalla dataset~\citep{cho2011friendship} contains 6,442,892 check-ins from 196,591 users at 1,280,969 locations,
which is distributed for academic research use. We used the check-in data in Austin, Texas.

\paragraph{Computational environment}
Our experiments were conducted on Intel(R) Xeon(R) Silver 4316 CPU @ 2.30GHz and 8 NVIDIA A100-SXM4-80GB GPUs.

\paragraph{Detailed setup for MovieLens experiments}
We computed the embeddings of movies using the top 64 principal components of the user-movie rating matrix on validation users.
We use top 250 movies (based on the validation set of users) as the candidate set for recommendation.
The privacy budget $\varepsilon$ is swept over $\{0.2, 0.5, 1.0, 2.0, 3.0, 4.0\}$.
The scale of the base measure for WPM and KPM is swept over the geometric grid in the feasible range $[e^{-\varepsilon/2} / k, e^{\varepsilon/2} / k]$.
We set $\lambda = 0.01$ and the number of iterations to $40$ for WPM, which is sufficient for convergence in our experiments.

\paragraph{Detailed setup for Gowalla experiments}
We use check-in data in Austin, Texas and drop users with fewer than 10 check-ins.
For WPM, we use the optimal base measure among uniform measures and set $\lambda = 0.01$ and the number of iterations to $40$,
which is sufficient for convergence in our experiments.
\section{Broader Impacts}\label{app:broader-impacts}
This work aims to develop sampling mechanisms with formal privacy guarantees, and thus has the potential to enhance privacy in various applications.
However, as with any privacy-preserving technology, there is a risk of misuse or unintended consequences.
For example, if the privacy parameters are not chosen carefully, it could lead to either insufficient privacy protection or excessive noise that degrades utility.


\end{document}